\documentclass[preprint, 12pt]{elsarticle}

\usepackage{gen_settings}

\journal{Artificial Intelligence}

\begin{document}

\begin{frontmatter}

\title{Sample-Based Bounds for Coherent Risk Measures: Applications to Policy Synthesis and Verification}

% \author[mymainaddress]{Anushri Dixit\corref{mycorrespondingauthor}}
% \cortext[mycorrespondingauthor]{Corresponding author}
% \ead{adixit@caltech.edu}
\author[1]{Prithvi Akella\corref{correspondingauthor}}
\cortext[correspondingauthor]{Corresponding author}
\ead{pakella@caltech.edu}

\author[1]{Anushri Dixit}
\ead{adixit@caltech.edu}

\author[1]{Mohamadreza Ahmadi}
\ead{mrahmadi@caltech.edu}

\author[1]{Joel W. Burdick}
\ead{jwb@robotics.caltech.edu}

\author[1]{Aaron D. Ames}
\ead{ames@caltech.edu}

\affiliation[1]{organization={California Institute of Technology},%Department and Organization
            city={Pasadena},
            country={USA}}
\begin{abstract}
The dramatic increase of autonomous systems subject to variable environments has given rise to the pressing need to consider risk in both the synthesis and verification of policies for these systems.  This paper aims to address a few problems regarding risk-aware verification and policy synthesis, by first developing a sample-based method to bound the risk measure evaluation of a random variable whose distribution is unknown.  These bounds permit us to generate high-confidence verification statements for a large class of robotic systems.  Second, we develop a sample-based method to determine solutions to non-convex optimization problems that outperform a large fraction of the decision space of possible solutions.  Both sample-based approaches then permit us to rapidly synthesize risk-aware policies that are guaranteed to achieve a minimum level of system performance.  To showcase our approach in simulation, we verify a cooperative multi-agent system and develop a risk-aware controller that outperforms the system's baseline controller.  We also mention how our approach can be extended to account for any $g$-entropic risk measure - the subset of coherent risk measures on which we focus.
\end{abstract}

\begin{keyword}
%% keywords here, in the form: keyword \sep keyword
Sampling Methods \sep Risk \sep Uncertain Systems \sep Safety-Critical Control \sep Control Barrier Functions \sep Policy Generation
%% PACS codes here, in the form: \PACS code \sep code
\PACS 0000 \sep 1111
%% MSC codes here, in the form: \MSC code \sep code
%% or \MSC[2008] code \sep code (2000 is the default)
\MSC 0000 \sep 1111
\end{keyword}

\end{frontmatter}

\section{Introduction}
The problem of optimal policy generation under uncertainty has been well-studied in the learning community, most notably via Reinforcement Learning~\cite{lillicrap2015continuous,kaelbling1996reinforcement,sutton2018reinforcement,mnih2013playing,bertsekas1995neuro}.  In this setting, the agent-environment interaction is modeled via a Markov Decision Process (MDP) where the state of the agent-environment pair is assumed to lie in some finite set~\cite{MDP}.  Then, the agent's action determines a set of states from which the final state is randomly sampled.  Each such transition is assigned a reward via a reward function, resulting in the traditional Reinforcement Learning policy generation problem - determine a policy that maximizes the expected, time-discounted reward achievable by the agent undergoing such uncertain transitions.  The Partially Observable Markov Decision formulation of this problem is considered in~\cite{monahan1982state,bhattacharya2020reinforcement,dung2008reinforcement,png2011bayesian}.

However, policies developed via expectation maximization may not be the most useful policies, especially when the variance in the outcome might yield problematic behavior~\cite{taha2011operations}.  This notion is especially true in the case of safety-critical control where robot safety in operation in uncertain environments is of paramount importance.  Consideration of such uncertainties has prompted a new vein of work that accounts for risk in the optimal generation of policies for these safety-critical systems.  We focus on the risk measures popularized by Artzner et al.~\cite{artzner1999coherent}.  These measures help discriminate between good and bad decisions insofar as those decisions that yield a more positive outcome of this measure are better than their counterparts in a risk-aware sense.  To that end then, there exist multiple techniques in the controls community to account for risk in the control generation procedure~\cite{majumdar2020should,singh2018framework,hakobyan2019risk,ahmadi2021risk}.  Likewise, there also exist a few techniques regarding risk-sensitive Reinforcement Learning~\cite{heger1994consideration,chow2017risk,mihatsch2002risk,geibel2005risk}.

While there exist a plethora of policy/control generation techniques in a risk-sensitive setting, there exist few verification techniques - especially for arbitrary risk measures - that account for unstructured uncertainty.  For example, there are numerous works detailing risk-aware verification procedures for specific systems~\cite{korb2003risk,vicentini2019safety,inam2018risk}.  These methods verify their systems of interest against existing widespread standards, \textit{e.g.} in~\cite{inam2018risk} the authors verify a multi-agent collaborative robotic system against the international standards for safe robot interactions with humans ISO 10218~\cite{isopart1,isopart2}.  As such, the verification analyses in these works are limited to their specific systems of interest, and the notion of risk is typically defined against the corresponding standard.  These works do not use the same notion of risk as utilized in control or policy development.  For more abstract, black-box approaches to verification, Corso et al. provide a very nice survey of existing techniques~\cite{corso2020survey}.  For the sake of brevity, we will not delve into each specific technique.  Rather, we will state that the existing verification techniques that account for uncertain system measurements (\textit{e.g.} Bayesian Optimization~\cite{deshmukh2017testing,mullins2018adaptive} or Reinforcement Learning~\cite{corso2019adaptive,koren2020adaptive}) follow the same expectation-specific analysis that prompted the interest in a more risk-aware approach.

\spacing
\newidea{Motivating Questions:} The inability to address risk-aware verification through the same risk measures utilized for policy development stems primarily from the inability to calculate these measures for unknown probability distributions.  While there exist concentration inequalities for specific coherent risk measures~\cite{thomas2019concentration,brown2007large,mhammedi2020pac,kagrecha2019distribution}, there do not exist similar bounds for other risk measures, \textit{e.g.} value at risk, entropic value at risk, \textit{etc}.  As such, the questions we aim to address are as follows.  First, can we develop sample-based bounds for arbitrary (coherent) risk measures and provide a fundamental requirement on the number of samples required to generate our bounds?  Second, can we use these bounds to provide high confidence statements on system performance in a risk-aware setting?  Furthermore, the risk-aware controller generation works cited prior typically require \textit{apriori} understanding of the underlying uncertainty whether via direct knowledge or in a distributionally-robust sense.  Granted, risk-aware Reinforcement Learning does not require such knowledge, and existing convergence bounds for the works cited prior guarantee that repeated iteration will eventually identify a satisfactory policy.  However, if we can determine a fundamental sample requirement for our risk measure bounds, and if the result of policy synthesis is the identification of satisfactory policies, can we provide similar sample requirements for the synthesis of satisfactory risk-aware policies and the relative complexity in identifying better policies?

\spacing
\newidea{Our Contributions:}  First, we develop a sample-based procedure that identifies a minimum sample requirement to upper bound the evaluation of a $g$-entropic risk measure for a random variable whose distribution is unknown.  Second, we rephrase risk-aware verification as a risk measure determination problem for a random variable whose distribution is unknown.  This permits us to use our prior results to generate high-confidence verification statements for arbitrarily complex robotic systems with limited system information.  Third, to facilitate the rapid synthesis of risk-aware policies, we develop a sample-based procedure to identify ``good" solutions to a large class of optimization problems, including the traveling salesman problem.  For this problem, by a ``good" solution we mean a path that is in the $99$-th percentile of all possible paths with respect to minimizing the distance traveled along the path.  By phrasing risk-aware synthesis as a similar type of optimization problem, we similarly rapidly identify ``good" risk-aware policies.  To showcase the efficacy of our risk-aware verification and synthesis results, we verify and synthesize a controller for a cooperative three-agent robotic system in its ability to avoid self-collisions while each agent traverses to its goal.  We also show that our synthesized controller outperforms the baseline controller with which the system is equipped by default, insofar as it more reliably satisfies both the safety-maintenance and goal-satisfaction objectives.

\spacing
\newidea{Structure:}

\secheader{Section~\ref{sec:prelims} - Preliminaries:} In this section, we provide a brief overview of some important topics - scenario optimization and $g$-entropic risk measures - and provide some general notation.

\secheader{Section~\ref{sec:concentration_inequalities} - Sample-Based Bounds for Risk Measures:} In this section, we detail our sample-based upper-bounding procedure for $g$-entropic risk measures and provide a fundamental sample requirement for their generation.

\secheader{Section~\ref{sec:verification} - Risk-Aware Verification:} In this section, we detail our risk-aware verification approach as a specific example of Section~\ref{sec:concentration_inequalities}.  Section~\ref{sec:verification_examples} details the specific example of risk-aware verification for a cooperative, three-agent robotic system.

\secheader{Section~\ref{sec:decision_selection} - ``Good" Decision Selection:}  In this section, we detail our sampling method to determine ``good" solutions to a large class of optimization problems.  Section~\ref{sec:decision_examples} details the application to the traveling salesman problem, wherein we generate a path that is in the $99$-th percentile of all possible paths with respect to minimizing the travel distance along the path.

\secheader{Section~\ref{sec:policy_synthesis} - Risk-Aware Synthesis:}  Finally, in this section we detail our risk-aware policy synthesis approach as a specific example of Sections~\ref{sec:verification} and~\ref{sec:decision_selection}.  Section~\ref{sec:synthesis_examples} details our risk-aware policy synthesis example for the same cooperative three-agent system verified in Section~\ref{sec:verification_examples} and shows that our synthesized policy outperforms the baseline controller.

\begin{figure}
    \centering
    \includegraphics{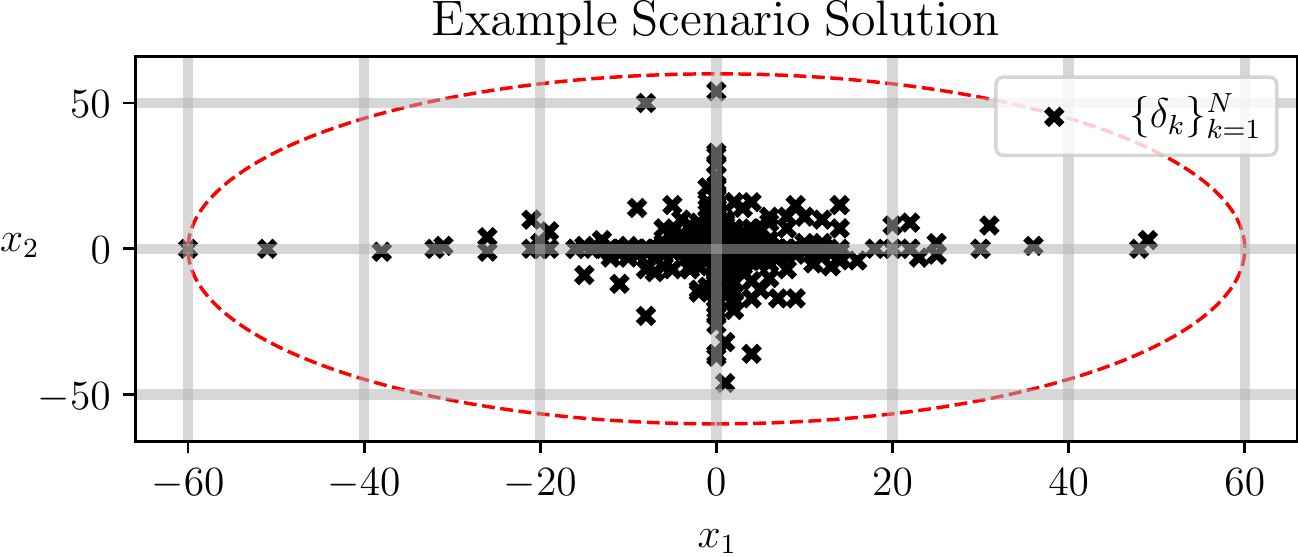}
    \caption{An example of using scenario optimization to calculate the radius of the largest required circle to encapsulate at least $0.987$ of the probability mass of the underlying random variable with which samples were taken.  This is a recreation of one of the examples done in~\cite{campi2008exact}.}
    \label{fig:scenario_ex}
\end{figure}

\section{Preliminaries}
\label{sec:prelims}
For the analysis to follow, it will be useful to provide a brief overview of two topics - scenario optimization and $g$-entropic risk measures - as they will be heavily utilized.  Before providing this overview, we will briefly describe some general notation.

\spacing
\newidea{General Notation:} $|A|$ denotes the cardinality of the set $A$.  $2^A$ denotes the set of all subsets of $A$. $\mathbb{Z}_+ = \{1,2,\dots\}$ is the set of all positive integers, and $\mathbb{R}_{\geq 0} = \{x\in\mathbb{R}~|~x \geq 0\}$.

\newidea{Scenario Optimization:}  The information in this section comes primarily from \cite{campi2008exact}.   Scenario optimization identifies robust solutions to uncertain convex optimization problems of the following form:
\begin{equation}
    \label{eq:uncertain_program}
    \tag{UP}
    \begin{aligned}
        z^* & = \argmin_{z \in \mathbb{Z} \subset \mathbb{R}^d}~ & &c^Tz, \\
        &~~\mathrm{subject~to}~ & &z \in \mathbb{Z}_{\delta},~\delta \in \Delta,
    \end{aligned}
\end{equation}
where $\mathbb{Z}$ is some convex subset of $\mathbb{R}^d$, $c \in \mathbb{R}^d$, $\delta$ is a sample of some random variable $W: \Omega \to \Delta$ with density function $\pi_{\delta}$, and $\mathbb{Z}_{\delta}$ is a convex subset of $\mathbb{R}^d$ that changes based on the received sample $\delta \in \Delta$. As $\Delta$ is typically a set of infinite cardinality, \textit{i.e.} $|\Delta| = \infty$, brute force identification of a decision $z^*$ such that $z^* \in \mathbb{Z}_{\delta},~\forall~\delta \in \Delta$ is usually infeasible.

To resolve this issue, scenario optimization solves a related optimization problem formed from an $N$-sized sample $\deltaset$ of the constraints $\delta$ and provides a probabilistic guarantee on the robustness of the corresponding solution $z^*_N$.  Specifically, given the sample scenario set $\deltaset$, we can construct the following scenario program for which we assume solvability for any $N$-sample set $\deltaset$:
\begin{equation}
    \label{eq:scenario_program}   
    \tag{RP-N}
    \begin{aligned}
        z^*_N & = \argmin_{z \in \mathbb{Z} \subset \mathbb{R}^d}~ & &c^Tz, \\
        &~~\mathrm{subject~to}~ & &z \in \mathbb{Z}_{\delta},~\forall~\delta \in \{\delta_k\}_{k=1}^N.
    \end{aligned}
\end{equation}
\begin{assumption}
\label{assump:RPN_solvability}
The scenario program~\eqref{eq:scenario_program} is solvable for any $N$-sample set $\deltaset$ and has a unique solution $z^*_N$.
\end{assumption}

However, as $z^*_N$ is the solution to~\eqref{eq:scenario_program}, there must exist a probability of sampling a constraint $\delta$ such that $z^*_N$ is not in the corresponding constraint set $\mathbb{Z}_{\delta}$.  Called the \emph{violation probability}, its definition is as follows.
\begin{definition}
\label{def:violation}
The \textit{violation probability} $V(z)$ of a given  decision $z \in \mathbb{Z}$ is defined as the probability of sampling a constraint $\delta$ to which $z$ is not robust, \textit{i.e.} $V(z) = \prob_{\pi_\delta}[\delta~|~z \not \in \mathbb{Z}_{\delta}]$ .
\end{definition}

\noindent Then, the main result in scenario optimization upper-bounds the violation probability with high confidence.
\begin{theorem}[Adapted from Theorem 1 in~\cite{campi2008exact}]
\label{thm:scenario_opt}
Let~\eqref{eq:scenario_program} be the scenario program for~\eqref{eq:uncertain_program} formed from the $N$-sample set $\{\delta_k\}_{k=1}^N$ of the constraints $\delta$ which are samples of a random variable with distribution $\pi_{\delta}$.  Furthermore, let $z^*_N \in \mathbb{R}^d$ be the solution to this scenario program~\eqref{eq:scenario_program}.  If Assumption~\ref{assump:RPN_solvability} holds, then $\forall~\epsilon\in[0,1]$,
\begin{equation}
    \prob^N_{\pi_\delta}[V(z^*_N) > \epsilon] \leq \sum_{i=0}^{d-1} \binom{N}{i} \epsilon^i(1-\epsilon)^{N-i}.
\end{equation}
\end{theorem}
\noindent Effectively, Theorem~\ref{thm:scenario_opt} bounds the probability that our scenario solution's violation probability $V(z^*_N)$ exceeds a cutoff $\epsilon \in [0,1]$ with respect to the induced probability distribution $\mathbb{P}_{\pi_{\delta}}^N$ from taking $N$ samples of the constraints $\delta$.   This bound is based on the number of samples taken, $N$, the dimension $d$ of our solution $z^*_N$, and the cutoff value $\epsilon$.  An example solution to a scenario optimization problem is shown in Figure~\ref{fig:scenario_ex}.

\spacing
\newidea{Risk Measures:} The information in this section will be adapted from the works of Ahmadi-Javid~\cite{ahmadi2012entropic} and Artzner et al.~\cite{artzner1999coherent}.  Risk measures $\phi$ map scalar random variables $X$ to the real-line.  More accurately, consider a probability space $(\Omega, \mathcal{F},P)$ with $\Omega$ the sample space, $\mathcal{F}$ the event space, and $P$ a probability measure.  A scalar random variable (R.V.) $X$ is a mapping from the sample space to the real-line, \textit{i.e.} $X:\Omega \to \mathbb{R}$.  The space of all such random variables $\Tilde{X} = \{X~|~X:\Omega \to \mathbb{R}\}$.  For $p \geq 1$, we define $L_p$ as the space of all random variables with $p$-bounded expectation, \textit{i.e.} $L_p = \{X~|~\expect[|X|^p] < \infty\}$, and $L_{\infty}$ is the set of all bounded scalar random variables.  Then, a risk measure $\phi$ is a function that maps from a subset $\mathbf{X} \subseteq \Tilde{X}$ to the extended real-line $\mathbb{R} \cup \{-\infty,\infty\}$, \textit{i.e.} $\phi: \mathbf{X} \to \mathbb{R} \cup \{-\infty,\infty\}$.  A coherent risk measure as per~\cite{artzner1999coherent} then, is defined as follows.
\begin{definition}
\label{def:coherent_risk}
A \textit{coherent risk measure} $\phi: \mathbf{X} \subseteq \Tilde{X} \to \mathbb{R} \cup \{-\infty,\infty\}$ satisfies the following four properties:
\begin{enumerate}
    \item Translation Invariance: $\phi(X+c) = \phi(X) + c$,
    \item Sub-Additivity: $\phi(X_1+X_2) \leq \phi(X_1) + \phi(X_2),~X_1,X_2 \in \mathbf{X}$,
    \item Monotonicity: If $X_1,X_2 \in \mathbf{X}$ and $X_1(\omega) \leq X_2(\omega)~\forall~\omega \in \Omega$, then $\phi(X_1) \leq \phi(X_2)$,
    \item Positive Homogeneity: $\phi(\lambda X) = \lambda \phi(X),~\forall~X \in \mathbf{X},~\lambda \geq 0$.
\end{enumerate}
\end{definition}
\noindent Furthermore, as mentioned in Theorem 3.2 in~\cite{ahmadi2012entropic}, every coherent risk measure satisfying the Fatou property (Definition 3.1 in~\cite{delbaen2002coherent}) has a distributionally robust dual formulation. 
\begin{theorem}
[Adapted from Theorem 3.2 in~\cite{ahmadi2012entropic}]
\label{thm:risk_dual_representation}
Let $\phi:L_{\infty} \to \mathbb{R}$ be a risk measure satisfying the Fatou property.  There exists a set $\mathcal{P}$ of probability measures on the sample and event spaces $(\Omega,\mathcal{F})$ such that the following equality is true if and only if $\phi$ is a coherent risk measure:
\begin{equation}
    \phi(X) = \sup_{Q \in \mathcal{P}}~\expect_{Q}[X].
\end{equation}
\end{theorem}
Then, Ahmadi-Javid~\cite{ahmadi2012entropic} defines $g$-entropic risk measures as those risk measures satisfying the condition in Theorem~\ref{thm:risk_dual_representation} with respect to a convex function $g$.
\begin{definition}
[Adapted from Definition 5.1 in~\cite{ahmadi2012entropic}]
\label{def:g-entropic}
Let $g$ be a convex function with $g(1) = 0$ and $\beta \geq 0$.  The \textit{$g$-entropic risk measure with divergence level $\beta$} $\entrisk_{g,\beta}$ is defined as follows, with $P$ the probability measure for $X$, ``$Q \ll P$" denoting $Q$ is absolutely continuous with respect to $P$, and $\frac{dQ}{dP}$ the Radon-Nikodym derivative:
\begin{equation}
    \entrisk_{g,\beta}(X) \triangleq \sup_{Q \in \mathcal{P}}~\expect_Q[X], \quad \mathcal{P} = \left\{Q~\bigg|~Q \ll P,~\int g\left(\frac{dQ}{dP}\right)dP \leq \beta\right\}.
\end{equation}
\end{definition}

\begin{figure}[t]
    \centering
    \includegraphics{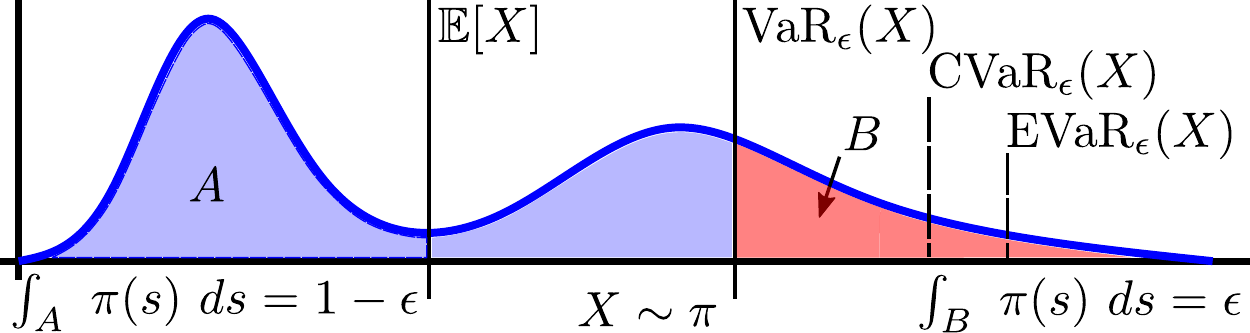}
    \caption{Example of three common risk measures - Value-at-Risk $\var_{\epsilon}(X)$, Conditional-Value-at-Risk $\cvar_{\epsilon}(X)$, and Entropic-Value-at-Risk $\evar_{\epsilon}(X)$ - for a scalar random variable $X$.  Per their definitions in Section~\ref{sec:prelims}, $\var_{\epsilon}(X) \leq \cvar_{\epsilon}(X) \leq \evar_{\epsilon}(X)$ as shown above.}
    \label{fig:risk-measures}
\end{figure}

\noindent All such $g$-entropic risk measures are coherent risk measures as per Theorem~\ref{thm:risk_dual_representation}~\cite{ahmadi2012entropic}.  Of more immediate use, however, will be their representation as infimum problems as mentioned in the following theorem.
\begin{theorem}
[Adapted from Theorem 5.1 in~\cite{ahmadi2012entropic}]
\label{thm:conjugate_dual_risk}
Let $g$ be a closed convex function with $g(1) = 0$ and $\beta \geq 0$.  For $X\in L_{\infty}$ and $\entrisk_{g,\beta}$ a $g$-entropic risk measure as per Definition~\ref{def:g-entropic}, the following equivalency holds with $g^*$ the convex-conjugate of $g$:
\begin{equation}    
    \label{eq:inf_g_entropic}
    \entrisk_{g,\beta}(X) = \inf_{t > 0, \mu \in \mathbb{R}}~t \left[\mu + \expect_P\left[g^*\left( \frac{X}{t} - \mu + \beta\right)\right] \right]
\end{equation}
\end{theorem}

Finally, a few common risk measures - Value-at-Risk, Conditional-Value-at-Risk, and Entropic-Value-at-Risk - are defined below.  Notably, Value-at-Risk is not a coherent risk measure whereas Conditional-Value-at-Risk and Entropic-Value-at-Risk are.  For all definitions, $X$ is a random variable with distribution $\pi$ and samples $x$.
\begin{definition}
\label{def:var}
The \emph{Value-at-Risk} level $\epsilon \in [0,1]$ denoted as $\var_{\epsilon}(X)$ is the infimum over $\zeta \in \mathbb{R}$ of $\zeta$ such that samples $x$ of $X$ lie below $\zeta$ with probability greater than or equal to $1-\epsilon$, \textit{i.e.}
\begin{equation}
    \var_{\epsilon}(X) = \inf \{\zeta~|~\prob_{\pi}[x \leq \zeta] \geq 1-\epsilon]\}.
\end{equation}
\end{definition}
\begin{definition}
\label{def:cvar}
The \textit{Conditional-Value-at-Risk} level $\epsilon \in (0,1]$ denoted as $\cvar_{\epsilon}(X)$ is the expected value of all samples $x$ of $X$ that are greater than or equal to the Value-at-Risk level $\epsilon$ for $X$, \textit{i.e.}
\begin{equation}
    \cvar_{\epsilon}(X) = \mathbb{E}_{\pi}[x~|~x \geq \var_{\epsilon}(X)] = \inf_{z \in \mathbb{R}}~z + \frac{\mathbb{E}_{\pi}\left[\max(X-z,0)\right]}{\epsilon}.
\end{equation}
\end{definition}
\begin{definition}
\label{def:evar}
The \textit{Entropic-Value-at-Risk} level $\epsilon \in (0,1]$ denoted as $\evar_{\epsilon}(X)$ is defined as the infimum over $z > 0$ of the Chernoff bound for $X$, \textit{i.e.}
\begin{equation}
    \evar_{\epsilon}(X) = \inf_{z > 0}~\frac{1}{z} \ln\left(\frac{\mathbb{E}_{\pi}\left[e^{zX}\right]}{\epsilon}\right)
\end{equation}
\end{definition}
\noindent Finally, the relationship between the three risk measures is shown in Figure~\ref{fig:risk-measures}.  Notably, for a random variable $X$ and some risk level $\epsilon \in (0,1)$, $\var_{\epsilon}(X) \leq \cvar_{\epsilon}(X) \leq \evar_{\epsilon}(X)$ as shown in~\cite{ahmadi2012entropic}.

\section{A Scenario Approach to Bounding g-Entropic Risk Measures}
\label{sec:concentration_inequalities}
To provide a general procedure for upper bounding $g$-entropic risk measures (which are coherent risk measures themselves), we will make use of Theorem~\ref{thm:conjugate_dual_risk}.  By equation~\eqref{eq:inf_g_entropic} specifically, we note that every $g$-entropic risk measure can be represented as the infimum of a loss function that linearly depends on the expectation of a function $g^*$ of the random variable $X$ of interest.  As a result, if we could identify a general procedure to upper bound the expectation of a function of a random variable of interest, we could exploit such a procedure to develop upper bounds on a large class of $g$-entropic risk measures.  This expectation upper bounding is the key step in our approach and will be formalized through the following problem statement.
\begin{problem}
\label{prob:ub_expect}
For a scalar random variable $X$ with samples $x$ and unknown distribution $\pi$, and a function $f:\mathbb{R} \to \mathbb{R}$, determine a method to upper bound the expected value of $f(x)$, \textit{i.e.} to upper bound $\expect_{\pi}[f(X)]$.
\end{problem}
This section will build on our prior work~\cite{akella2022scenario} from which we will utilize an optimization problem and theorem.  The optimization problem is below, where $\{\ell_i\}_{i=1}^N$ is a set of scalar values $\ell \in \mathbb{R}$ which will be defined whenever we require the solution to~\eqref{eq:upper_bound_scenario}, \textit{e.g.} we will state $\zeta^*_N$ is the solution to~\eqref{eq:upper_bound_scenario} \textit{for a set of samples $\{\ell_k\}_{k=1}^N$} and specify from where the samples originated:
\begin{equation}
    \label{eq:upper_bound_scenario}
    \tag{UB-RP-N}
    \begin{aligned}
        \zeta^*_N & = \argmin_{\zeta \in \mathbb{R}}~ & &\zeta, \\
        &~~\mathrm{subject~to}~ & & \zeta \geq \ell_i,~\forall~\ell_i \in \{\ell_k\}_{k=1}^N.
    \end{aligned}
\end{equation}
Then, the theorem we proved regarding high-confidence upper bounding of the Value-at-Risk level $\epsilon$ for a random variable $X$ is as follows.
\begin{theorem}[Adapted from Theorem 2 in~\cite{akella2022scenario}]
\label{thm:bounding_var}
Let $\zeta^*_N$ be the solution to~\eqref{eq:upper_bound_scenario} for a set of $N$ samples $\{x_k\}_{i=k}^N$ of a random variable $X$ with unknown distribution function $\pi$.  The following statement is true $\forall~\epsilon \in[0,1]$ and with $\var_{\epsilon}(X)$ as defined in Definition~\ref{def:var}:
\begin{equation}
    \prob^N_{\pi}[\zeta^*_N \geq \var_{\epsilon}(X)] \geq 1-(1-\epsilon)^N.
\end{equation}
\end{theorem}

\subsection{Upper Bounding the Expected Value of Functions of a Random Variable}
Upper bounding the expectation of functions of a random variable then requires two steps beyond our prior work.  First, we note that if the R.V. $X$ has an upper bound $u_b \in \mathbb{R}$, \textit{i.e.} $\prob_{\pi}[x~|~x \leq u_b] = 1$, then we can always upper bound the expected value of $X$ provided a cutoff $c \in \mathbb{R}$ and the probability mass of samples $x$ of $X$ that are below this cutoff:
\begin{equation}
    \label{eq:upper_bnd_expectation}
    \expect_{\pi}[X] = \int_{\mathbb{R}}~s\pi(s)~ds \leq c\prob_{\pi}[x~|~x\leq c] + u_b(1-\prob_{\pi}[x~|~x\leq c]).
\end{equation}
Identification of such a cutoff $c$ and bounding of the probability mass to its left $\prob_{\pi}[x~|~x\leq c]$ are both outcomes of Theorem~\ref{thm:bounding_var}.  As such, we can use our prior work to identify an upper bound to $\expect_{\pi}[X]$ \textit{without knowledge of its distribution} $\pi$ provided we know a (perhaps) loose upper bound $u_b \in \mathbb{R}$.

\begin{figure}[t]
    \centering
    \includegraphics{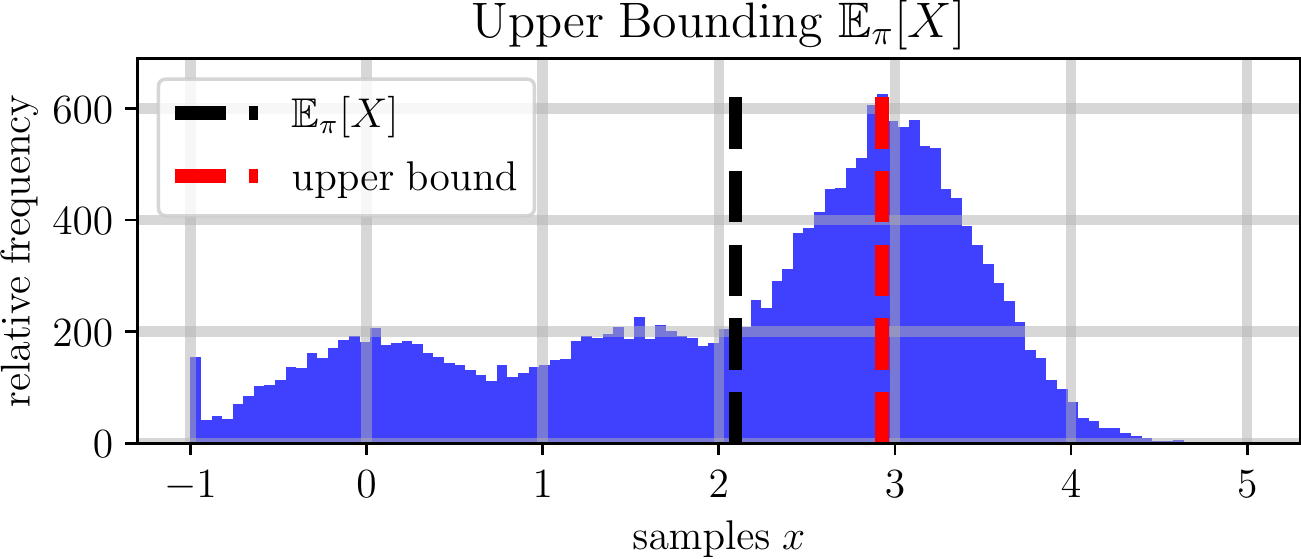}
    \caption{To upper bound $g$-entropic risk measures, we motivate that we first must be able to upper bound the expected value of a scalar random variable $X$.  The figure above is an example of our ability to do this as per Theorem~\ref{thm:bound_expectation}.  Here, we upper bound (red) the expected value (black) of a multi-modal random variable $X$ by taking $N = 20$ samples of $X$ and knowing its upper bound $u_b = 5$.  The true distribution (blue) was calculated numerically by taking $20000$ samples of $X$.}
    \label{fig:ub_expectation}
\end{figure}

\begin{theorem}
\label{thm:bound_expectation}
Let $X$ be a scalar R.V. with samples $x$, unknown distribution function $\pi$, and upper bound $u_b \in \mathbb{R}$ such that $\prob_{\pi}[x~|~x \leq u_b] = 1$.  Furthermore, let $\zeta^*_N$ be the solution to~\eqref{eq:upper_bound_scenario} for a set of $N$ samples $\{x_k\}_{k=1}^N$ of $X$.  The following statement is true $\forall~\epsilon \in[0,1]$:
\begin{equation}
    \prob^N_{\pi}\left[\expect_{\pi}[X] \leq \zeta^*_N(1-\epsilon) + u_b\epsilon\right] \geq 1-(1-\epsilon)^N.
\end{equation}
\end{theorem}
\begin{proof}
First, by Theorem~\ref{thm:bounding_var} and the definition of $\var_{\epsilon}(X)$ in Definition~\ref{def:var}, we can lower bound with high confidence the probability with which samples $x$ of $X$ lie below $\zeta^*_N$:
\begin{equation}
    \prob^N_{\pi}\left[\prob_{\pi}[x~|~x \leq \zeta^*_N] \geq 1-\epsilon \right] \geq 1-(1-\epsilon)^N.
\end{equation}
Then by equation~\eqref{eq:upper_bnd_expectation} we have the following inequality:
\begin{equation}
    \label{eq:t3_step_before_end}
    \expect_{\pi}[X] \leq \zeta^*_N \prob_{\pi}[x~|~x \leq \zeta^*_N] + u_b(1-\prob_{\pi}[x~|~x\leq\zeta^*_N).
\end{equation}
Also, by Theorem~\ref{thm:bounding_var} and the definition of $\var_{\epsilon}(X)$ in Definition~\ref{def:var}, we know that $\prob_{\pi}[x~|~x \leq \zeta^*_N] \geq 1-\epsilon$ with some minimum probability. To finish the proof we note that the right hand side of the inequality in~\eqref{eq:t3_step_before_end} is maximized when this probability $\prob_{\pi}[x~|~x \leq \zeta^*_N]$ equals its lower bound $1-\epsilon$, as $u_b \geq \zeta^*_N$.  As this lower bound holds with minimum probability $1-(1-\epsilon)^N$, the result holds with the same probability, \textit{i.e.}
\begin{equation}
    \prob^N_{\pi}\left[\expect_{\pi}[X] \leq \zeta^*_N(1-\epsilon) + u_b\epsilon\right] \geq 1-(1-\epsilon)^N.
\end{equation}
\end{proof}

The second step we require to bound arbitrary functions of our random variable $X$ is noting that if we have a function $f:\mathbb{R} \to \mathbb{R}$, then $f(X) = Y$ is another scalar random variable with samples $y$ and distribution $\pi_Y$.  As a result, we can make use of the prior theorem bounding the expectation of a random variable - $Y$ in this case - and such an upper bound would naturally correspond to an upper bound on $\expect_{\pi}[f(X)]$.  This is formally expressed in Corollary~\ref{corr:ub_func_expect}.
\begin{corollary}
\label{corr:ub_func_expect}
Let $X$ be a scalar R.V. with samples $x$ and distribution $\pi$.  Let $f:\mathbb{R} \to \mathbb{R}$ and let $Y = f(X)$ be another scalar R.V. with samples $y=f(x)$, distribution $\pi_Y$, and upper bound $u_b$ such that $\prob_{\pi_Y}[y~|~y \leq u_b] = 1$.  Furthermore, let $\zeta^*_N$ be the solution to~\eqref{eq:upper_bound_scenario} for a set of $N$ samples $\{y_k = f(x_k)\}_{k=1}^N$ of $Y$.  The following statement is true $\forall~\epsilon \in [0,1]$.
\begin{equation}
    \label{eq:ub_func_expect}
    \prob^N_{\pi}\left[\expect_{\pi}[f(X)] \leq \zeta^*_N(1-\epsilon) + u_b\epsilon\right] \geq 1-(1-\epsilon)^N.
\end{equation}
\end{corollary}
\begin{proof}
Via Theorem~\ref{thm:bound_expectation} applied to the R.V. $Y$ with samples $y$, we have the following statement:
\begin{equation}
    \prob^N_{\pi_Y}\left[\expect_{\pi_Y}[Y] \leq \zeta^*_N(1-\epsilon) + u_b\epsilon\right] \geq 1-(1-\epsilon)^N.
\end{equation}
As mentioned in the assumptions though, we generated samples $y$ of $Y$, by applying $f$ to samples $x$ of $X$.  As such, the above inequality is equivalent to~\eqref{eq:ub_func_expect} completing the proof.
\end{proof}

\begin{figure}[t]
    \centering
    \includegraphics{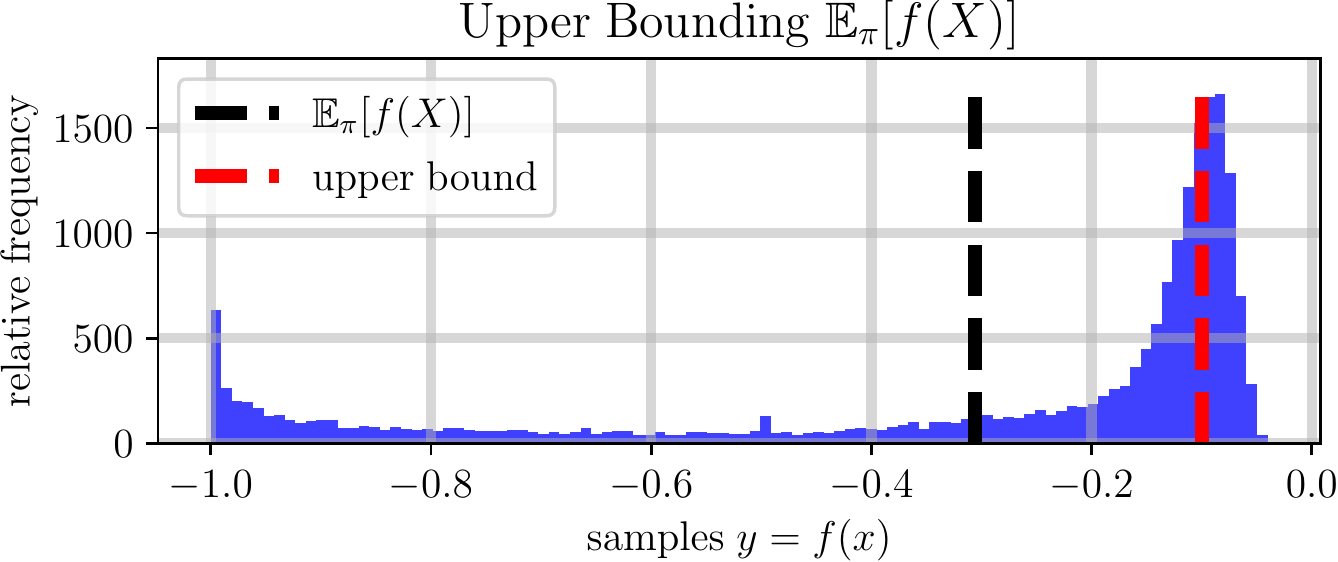}
    \caption{The second step to upper bounding $g$-entropic risk measures is upper bounding the expected value of a function $f$ of a scalar random variable $X$ with unknown distribution $\pi$.  Above is an example of us upper bounding $\expect_{\pi}[f(X)]$ where $f(x) = \frac{-1}{x^2+1}$ for the multi-modal random variable $X$ whose expectation we upper bounded in Figure~\ref{fig:ub_expectation}.  To generate this upper bound we took $N = 20$ samples of the random variable $Y=f(X)$.  We formally prove our ability to do this upper bound in Corollary~\ref{corr:ub_func_expect}.  As before, the true expected value is shown in black and the upper bound is shown in red.}
    \label{fig:ub_func_expectation}
\end{figure}

\subsection{Upper Bounding g-entropic risk measures}
With Corollary~\ref{corr:ub_func_expect}, we can develop our procedure to upper bound a large class of $g$-entropic risk measures, including Conditional-Value-at-Risk and Entropic-Value-at-Risk.  The main idea is as follows.  First, for the $g$-entropic risk measure of interest, $\entrisk_{g,\beta}$ as per Definition~\ref{def:g-entropic}, we require that the convex conjugate $g^*$ for $g$ maps from the reals to the reals, \textit{i.e.} $g^*:\mathbb{R} \to \mathbb{R}$, and also has an upper bound when applied to samples $x$ of the random variable $X$.  We will formally state this in the following assumption.
\begin{assumption}
\label{assump:g_measure_bound}
$X$ is a scalar random variable with samples $x$, distribution $\pi$, and upper bound $\ell \in \mathbb{R}$ such that $\prob_{\pi}[x~|~x \leq \ell] = 1$. $\entrisk_{g,\beta}$ is a $g$-entropic risk measure as per Definition~\ref{def:g-entropic} with respect to some $\beta \geq 0$ and closed convex function $g$ such that $g(1) = 0$. $g^*:\mathbb{R} \to \mathbb{R}$ is the convex-conjugate for $g$, and $g^*$ has an upper bounding function $u_b: D \subset \mathbb{R}^2 \to \mathbb{R}$ satisfying the following probabilistic equality $\forall~\mu \in \mathbb{R},~t > 0$:
\begin{gather}
    L(x,\mu,t) \triangleq t\left(\mu + g^*\left(\frac{x}{t}-\mu + \beta\right)\right), \label{eq:inf_deterministic_loss}\\
    \prob_{\pi}\left[x~\big|~L(x,\mu,t) \leq u_b(\mu,t) \right] = 1. \label{eq:inf_deterministic_ub}
\end{gather}
Finally, $\zeta^*_N(\mu,t)$ is the solution to~\eqref{eq:upper_bound_scenario} for a set of $N$-samples $\{y_k = L(x_k,\mu,t)\}_{k=1}^N$ of the random variable $Y = L(X,\mu,t)$ for some $\mu \in \mathbb{R},~t>0$.
\end{assumption}
\noindent For context, this assumption is not too restrictive, as it is easily satisfied by both $\cvar$ and $\evar$ for any risk level $\alpha \in (0,1]$, and we will show this to be the case.  The second step then, is to upper bound the objective function in~\eqref{eq:inf_g_entropic}, via Corollary~\ref{corr:ub_func_expect}.  This will change~\eqref{eq:inf_g_entropic} from an optimization problem over an (unknown) distribution $\pi$ for the random variable $X$, to an optimization problem over a set of sampled values $\{x_k\}_{k=1}^N$ of $X$ which is easily solvable.  The formal statement of this procedure is proven in the following lemma,  theorem, and corollary.  The lemma will state that we can upper bound the expected value of the convex conjugate $g^*$ applied to samples $x$ of $X$ and the theorem will utilize this lemma to provide a high confidence upper bound on the $g$-entropic risk measure of interest.  The corollary will formalize a relationship between the confidence in our estimate and the number of samples we take of the random variable.

\begin{lemma}
\label{lem:bound_inf_objective}
Let Assumption~\ref{assump:g_measure_bound} hold and let $L,u_b$ as defined in equations~\eqref{eq:inf_deterministic_loss} and~\eqref{eq:inf_deterministic_ub} respectively.  Then, $\forall~\epsilon \in[0,1],~\mu\in\mathbb{R},~t > 0$,
\begin{equation}
    \prob^N_{\pi}\left[\expect_{\pi}[L(X,\mu,t)] \leq \zeta^*_N(\mu,t)(1-\epsilon) + u_b(\mu,t)\epsilon \right] \geq 1-(1-\epsilon)^N.
\end{equation}
\end{lemma}
\begin{proof}
This is a direct application of Corollary~\ref{corr:ub_func_expect}.
\end{proof}

\begin{theorem}
\label{thm:ub_g_entropic}
Let Assumption~\ref{assump:g_measure_bound} hold.  Then, $\forall~\epsilon \in [0,1]$,
\begin{equation}
    \label{eq:ub_g_entropic}
    \prob^N_{\pi}\left[ER_{g,\beta}(X) \leq \inf_{t > 0,~\mu \in \mathbb{R}}~\zeta^*_N(\mu,t)(1-\epsilon) + u_b(\mu,t)\epsilon\right] \geq 1-(1-\epsilon)^N.
\end{equation}
\end{theorem}
\begin{proof}
Via Theorem~\ref{thm:conjugate_dual_risk}, the $g$-entropic risk measure $\entrisk_{g,\beta}$ can be represented via an infimum:
\begin{equation}    
    \label{eq:entropic_infimum_formulation}
    \entrisk_{g,\beta}(X) = \inf_{t > 0, \mu \in \mathbb{R}}~t \left[\mu + \expect_{\pi}\left[g^*\left( \frac{X}{t} - \mu + \beta\right)\right] \right].
\end{equation}
Then, via linearity of the expectation operator and the definition of $L$ in~\eqref{eq:inf_deterministic_loss}, we can rewrite~\eqref{eq:entropic_infimum_formulation} as follows:
\begin{equation}
     \entrisk_{g,\beta}(X) = \inf_{t > 0, \mu \in \mathbb{R}}~\expect_{\pi}[L(X,\mu,t)].
\end{equation}
Then the result holds via Lemma~\ref{lem:bound_inf_objective}.
\end{proof}

\begin{corollary}
\label{corr:fund_sample_requirement}
Let Assumption~\ref{assump:g_measure_bound} hold, $\gamma \in [0,1)$, and $\epsilon \in (0,1)$.  If $\zeta^*_N$ is the solution to~\eqref{eq:upper_bound_scenario} for a set of $N$ samples $\{x_k\}_{k=1}^N$ of the random variable $X$ where $N \geq \frac{\log(1-\gamma)}{\log(1-\epsilon)}$, then
\begin{equation}
    \prob^N_{\pi}\left[ER_{g,\beta}(X) \leq \inf_{t > 0,~\mu \in \mathbb{R}}~\zeta^*_N(\mu,t)(1-\epsilon) + u_b(\mu,t)\epsilon\right] \geq \gamma.
\end{equation}
\end{corollary}
\begin{proof}
This is a direct consequence of Theorem~\ref{thm:ub_g_entropic}, as
\begin{equation}
    N \geq \frac{\log(1-\gamma)}{\log(1-\epsilon)} \implies 1- (1-\epsilon)^N \geq \gamma.
\end{equation}
\end{proof}

To summarize, Theorem~\ref{thm:ub_g_entropic} states that we can upper bound with high confidence the $g$-entropic risk measure of a random variable $X$ with unknown distribution $\pi$ provided that we know an upper bounding function $u_b$ for a function $L$ of samples $x$ of the random variable $X$.  Corollary~\ref{corr:fund_sample_requirement} provides the minimum number of samples $N$ required to determine this upper bound with a pre-specified confidence $\gamma \in (0,1)$. Based on Corollary~\ref{corr:fund_sample_requirement}, it may seem that one could arbitrarily choose $\epsilon \in (0,1)$ to minimize the minimum sample requirement.  While theoretically possible, too large a value of $\epsilon$ will result in the upper bound $r^* = \ell$ - the upper bound for the random variable $X$.  In practice and in the examples to follow, choosing $\epsilon < 1-\gamma$ resolves this issue.  That being said, to generate this bound $r^*$ we must solve a non-convex optimization problem over samples of this random variable.  To show that this problem is tractable, we will produce upper bounds on both the Conditional-Value-at-Risk and Entropic-Value-at-Risk.

\subsection{Specializing to CVaR and EVaR}
Ahmadi-Javid~\cite{ahmadi2012entropic} identifies the convex conjugate function $g^*$ and parameter $\beta$ with which the Conditional-Value-at-Risk can be recast as a $g$-entropic risk measure.
\begin{remark}
\label{rem:cvar_entropic_params}
The Conditional-Value-at-Risk level $\alpha \in (0,1]$ can be recast as a $g$-entropic risk measure with convex conjugate function $g^*(x) = \frac{1}{\alpha}\max\{x,0\}$ and scalar parameter $\beta = 0$~\cite{ahmadi2012entropic}.
\end{remark}
\noindent Then, to use Theorem~\ref{thm:ub_g_entropic} we must show that the Conditional-Value-at-Risk satisfies Assumption~\ref{assump:g_measure_bound}.
\begin{lemma}
\label{lem:cvar_satisfies_assump}
The Conditional-Value-at-Risk for any risk-level $\alpha \in (0,1]$ satisfies Assumption~\ref{assump:g_measure_bound}.
\end{lemma}
\begin{proof}
To start, $\cvar$ for any risk-level $\alpha$ is a $g$-entropic risk measure with $g^*(x) = \frac{1}{\alpha}\max\{x,0\}$ and $\beta = 0$.  As a result, a solution $\zeta^*_N(\mu,t)$ will always exist for~\eqref{eq:upper_bound_scenario} as it is the solution to a linear program minimizing a scalar decision variable subject to a finite set of lower bounds taking values in $\mathbb{R}$.  Then, to prove that the Conditional-Value-at-Risk at any risk level $\alpha \in (0,1]$ satisfies Assumption~\ref{assump:g_measure_bound}, it suffices to identify an upper bounding function $u_b: D \subset \mathbb{R}^2 \to \mathbb{R}$ under the assumption that the scalar random variable $X$ with distribution $\pi$ and samples $x$ has an upper bound $\ell \in \mathbb{R}$ such that $\prob_{\pi}[x~|~x \leq \ell] = 1$.  This function $u_b$ will be defined as follows:
\begin{equation}
    \label{eq:cvar_ub_function}
    L(x,\mu,t) = t\left( \mu + \frac{1}{\alpha}\max\left\{\frac{x}{t} - \mu, 0\right\}\right), \quad u_b(\mu,t) = L(\ell,\mu,t).
\end{equation}
Since this upper bounding function satisfies the probabilistic equality in Assumption~\ref{assump:g_measure_bound}, the Conditional-Value-at-Risk for any risk-level $\alpha \in (0,1]$ satisfies Assumption~\ref{assump:g_measure_bound}.
\end{proof}

\noindent As the Conditional-Value-at-Risk satisfies Assumption~\ref{assump:g_measure_bound}, we can use $u_b$~\eqref{eq:cvar_ub_function} to provide high-confidence estimates on $\cvar_{\alpha}(X)~\forall~\alpha \in (0,1]$.

\begin{figure}[t]
    \centering
    \includegraphics{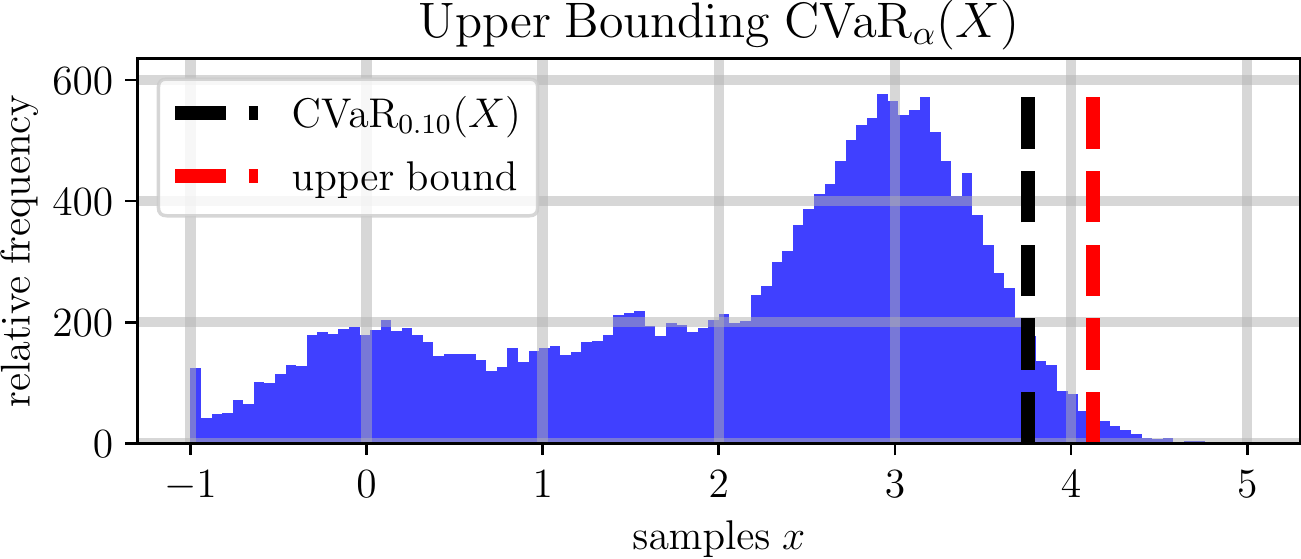}
    \caption{A specific $g$-entropic risk measure, the Conditional-Value-at-Risk for any risk-level $\alpha \in (0,1]$ should be upper-boundable via our scenario approach as expressed in Corollary~\ref{corr:ub_cvar}.  Shown above is an example of our scenario approach in bounding $\cvar_{\alpha = 0.1}(X)$ for the multi-modal random variable $X$ whose distribution is shown in blue .  To generate the upper bound shown in red, we required $N=45$ samples of $X$.  The true $\cvar_{0.1}(X)$ is shown in black.}
    \label{fig:ub_cvar}
\end{figure}

\begin{corollary}
\label{corr:ub_cvar}
Let $X$ be a scalar random variable with samples $x$, distribution $\pi$, and upper bound $\ell \in \mathbb{R}$ such that $\prob_{\pi}[x~|~x \leq \ell] = 1$.  Let $\alpha \in (0,1],~\epsilon\in[0,1]$, and $L,u_b$ be as defined in~\eqref{eq:cvar_ub_function} with respect to this upper bound $\ell$ and constant $\alpha$.  Furthermore, let $\zeta^*(\mu,t)$ be the solution to~\eqref{eq:upper_bound_scenario} for a set of $N$-samples $\{y_k = L(x_k,\mu,t)\}_{k=1}^N$ of the random variable $Y = L(X,\mu,t)$.  Then,
\begin{equation}
    \prob^N_{\pi}\left[\cvar_{\alpha}(X) \leq \inf_{\mu \in \mathbb{R},~t>0}~\zeta^*_N(\mu,t)(1-\epsilon) + u_b(\mu,t)\epsilon \right] \geq 1-(1-\epsilon)^N.
\end{equation}
\end{corollary}
\begin{proof}
This is a direct application of Theorem~\ref{thm:ub_g_entropic} due to Lemma~\ref{lem:cvar_satisfies_assump}.
\end{proof}

\begin{figure}[t]
    \centering
    \includegraphics{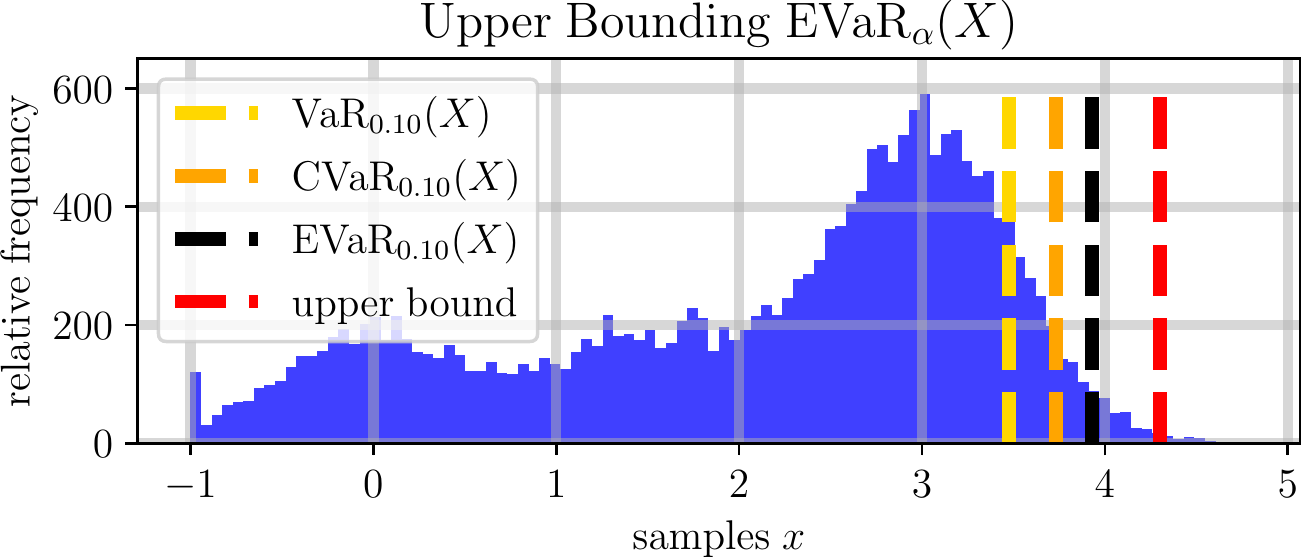}
    \caption{A culmination of our approach in devising a method to upper bound arbitrary $g$-entropic risk measures.  Shown above is our attempt to upper bound the Entropic-Value-at-Risk at risk level $\alpha = 0.1$ of the scalar multi-modal random variable $X$ whose distribution is shown in blue.  To calculate this upper bound (red) with $N = 20$ samples of $X$, we used the same method that we used to determine the upper bound for $\cvar_{\alpha}(X)$ in Figure~\ref{fig:ub_cvar}.  We formally state our capacity to do this in Corollary~\ref{corr:ub_evar}.  Notice that the true $\var_{\alpha}(X) \leq \cvar_{\alpha}(X) \leq \evar_{\alpha}(X)$ as also shown in Figure~\ref{fig:risk-measures}.  As before, the true $\evar_{\alpha}(X)$ is shown in black.}
    \label{fig:ub_evar}
\end{figure}

To use Theorem~\ref{thm:ub_g_entropic} to provide an upper bound on the Entropic-Value-at-Risk for any risk level $\alpha \in (0,1]$, we will follow a similar procedure as we followed for the Conditional-Value-at-Risk.  First, Ahmadi-Javid~\cite{ahmadi2012entropic} identifies the convex conjugate function $g^*$ and parameter $\beta$ which enables Entropic-Value-at-Risk to be cast as a $g$-entropic risk measure.
\begin{remark}
\label{rem:evar_entropic_params}
The Entropic-Value-at-Risk level $\alpha \in (0,1]$ can be recast as a $g$-entropic risk measure with convex conjugate function $g^*(x) = e^{x-1}$ and scalar parameter $\beta = - \ln(\alpha)$~\cite{ahmadi2012entropic}.
\end{remark}
\noindent Then, we note that the Entropic-Value-at-Risk for any risk-level $\alpha \in (0,1]$ satisfies Assumption~\ref{assump:g_measure_bound}.
\begin{lemma}
\label{lem:evar_satisfies_assump}
The Entropic-Value-at-Risk for any risk-level $\alpha \in (0,1]$ satisfies Assumption~\ref{assump:g_measure_bound}.
\end{lemma}
\begin{proof}
This proof follows the proof of Lemma~\ref{lem:cvar_satisfies_assump} insofar as it suffices to identify an upper bounding function $u_b: D \subset \mathbb{R}^2 \to \mathbb{R}$ satisfying the probabilistic inequality in Assumption~\ref{assump:g_measure_bound} where $g^*(x) = e^{x-1}$ and $\beta = -\ln(\alpha)$.  The following function $u_b$ suffices:
\begin{equation}
    \label{eq:evar_ub_function}
    L(x,\mu,t) = t\left(\mu + e^{\frac{x}{t} - \mu - \ln(\alpha) - 1} \right), \quad u_b(\mu,t) = L(\ell,\mu,t).
\end{equation}
Here, $\ell$ is the assumed upper bound on samples $x$ of the random variable $X$ with distribution $\pi$, \textit{i.e.} $\prob_{\pi}[x~|~x \leq \ell] = 1$:
\end{proof}

\noindent Since $\evar_{\alpha}~\forall~\alpha \in(0,1]$ satisfies Assumption~\ref{assump:g_measure_bound}, we can use $u_b$~\eqref{eq:cvar_ub_function} to provide high-confidence estimates on $\evar_{\alpha}(X)~\forall~\alpha \in (0,1]$.
\begin{corollary}
\label{corr:ub_evar}
Let $X$ be a scalar random variable with samples $x$, distribution $\pi$, and upper bound $\ell \in \mathbb{R}$ such that $\prob_{\pi}[x~|~x \leq \ell] = 1$.  Let $\alpha \in (0,1],~\epsilon \in [0,1]$, and $L,u_b$ be as defined in~\eqref{eq:evar_ub_function} with respect to this upper bound $\ell$ and constant $\alpha$.  Furthermore, let $\zeta^*_N(\mu,t)$ be the solution to~\eqref{eq:upper_bound_scenario} for a set of $N$-samples $\{y_k = L(x_k,\mu,t)\}_{k=1}^N$ of the random variable $Y = L(X,\mu,t)$.  Then,
\begin{equation}
    \prob^N_{\pi}\left[\evar_{\alpha}(X) \leq \inf_{\mu \in \mathbb{R},~t>0}~\zeta^*_N(\mu,t)(1-\epsilon) + u_b(\mu,t)\epsilon \right] \geq 1-(1-\epsilon)^N.
\end{equation}
\end{corollary}
\begin{proof}
This is an application of Theorem~\ref{thm:ub_g_entropic} via Lemma~\ref{lem:evar_satisfies_assump}.
\end{proof}

\section{Risk-Aware Verification}
\label{sec:verification}
An important application of the sample-based bounds derived in the prior section arises in safety-critical system verification where system evolution is partially stochastic due to unmodeled dynamics, noise, \textit{etc}.  In this section, we will detail how we can reformulate safety-critical system verification as a risk measure determination problem.  This reformulation lets us use our prior results to easily bound system performance in a cooperative multi-agent system setting as exemplified in Section~\ref{sec:verification_examples}.

\subsection{Notation and Problem Setting}
For our (nonlinear) system under study, $\mathcal{X}$ is the state space, $\mathcal{U}$ is the input space, and $\Theta$ is a known space of parameters $\theta$ influencing the system's controller $U$.  Furthermore, we assume the system is subject to stochastic noise $\xi$ with an unknown distribution $\pi_{\xi}(x,u,t)$ over $\mathbb{R}^n$.
\begin{align}
    \dot x & = f(x,u) + \xi, & & x \in \mathcal{X} \subset \mathbb{R}^n,~u \in \mathcal{U} \subset \mathbb{R}^m,~\label{eq:evolution}\\
    u & = U(x,\theta), & & \theta \in \Theta \subset \mathbb{R}^p, \\
    \xi & \sim \pi_{\xi}(x,u,t), & & \int_{\mathcal{X}}\pi_{\xi}(x,u,t,s)~ds = 1~\forall~x,u,t.
\end{align}
We denote $x^\theta_t$ as our closed-loop system solution at time $t$ - note that the parameter $\theta$ does not change over a trajectory - and $x^{\theta}$ will correspond to our closed-loop state signal:
\begin{equation}
    \label{eq:CL_sys}
    \dot x^\theta_t = f\left(x^\theta_t, U\left(x^\theta_t, \theta\right)\right) + \xi, \quad x^\theta \in \signalspace,~\signalspace = \{s:\mathbb{R}_{\geq 0} \to \mathbb{R}^n\}
\end{equation}

Verification work typically assumes the existence of a robustness metric $\rho$ - a function that maps state trajectory signals to the real line, with positive evaluations of the metric indicating system objective satisfaction~\cite{donze2010robust}.
\begin{definition}
\label{def:robustness}
A \textit{robustness metric} $\rho$ is a function $\rho: \signalspace \to [-a,b],~a,b \in \mathbb{R}_{++}$ such that $\rho(s) \geq 0$ only for those signals $s$ that exhibit desired properties.
\end{definition}
\noindent Examples of robustness metrics $\rho$ include the minimum value of a control barrier function $h$ over some pre-specified time horizon~\cite{ames2016control, xu2015robustness}, or the robustness metrics of Signal Temporal Logic~\cite{donze2010robust, raman2014model}.  As the existence and construction of these functions has been well-studied, we will simply assume their existence for the time being.

For risk-aware verification, the uncertainty arises through the uncertainty $\xi$ entering the system dynamics in~\eqref{eq:evolution}.  As a result, the robustness of a closed loop trajectory $\rho(x^\theta)$ is a scalar random variable $R(x_0,\theta)$ with some distribution $\pi_R(x_0,\theta)$ that depends on the initial condition $x_0$ and parameter $\theta$ for our closed-loop system trajectory $x^\theta$.  By further uniformly sampling initial conditions and parameters from their respective spaces \textit{i.e.} sampling $(x_0,\theta)$ from $\uniform[\mathcal{X}_0 \times \Theta]$, we generate the scalar randomized robustness variable $R$ with distribution $\pi_R$ that was the subject of study in~\cite{akella2022scenario}.
\begin{definition}
\label{def:randomized_robustness} The \textit{holistic system robustness},
$R$, is a scalar random variable with distribution $\pi_R$ and samples $r$ denoting the closed-loop robustness $\rho(x^\theta)$ of trajectories $x^\theta$ whose initial condition and parameter $(x_0,\theta)$ were sampled uniformly from $\mathcal{X}_0 \times \Theta$.
\end{definition}
\noindent Then the formal statement of the risk-aware verification problem will follow.
\begin{problem}
\label{prob:2}
For the scalar random variable $R$ with distribution $\pi_R$ as per Definition~\ref{def:randomized_robustness} devise a method to determine upper bounds $r_C^*, r_E^* \in \mathbb{R}$ with corresponding probabilities $\epsilon_C,\epsilon_E \in [0,1]$ such that for some $\alpha \in (0,1]$,
\begin{equation}
    \prob_{\pi_R}[r_C^* \geq \cvar_{\alpha}(-R)] \geq 1-\epsilon_C,~\prob_{\pi_R}[r_E^* \geq \evar_{\alpha}(-R)] \geq 1-\epsilon_E.
\end{equation}
Here, $\cvar_{\alpha}(-R),~\evar_{\alpha}(-R)$ are the coherent risk measures defined in Definitions~\ref{def:cvar} and~\ref{def:evar} respectively.
\end{problem}

\subsection{Upper Bounds for Risk-Aware Verification}
First, we note that samples $r$ of the holistic system robustness $R$ are defined as the robustness of a randomly sampled system trajectory, \textit{i.e.} $r = \rho(x^\theta)$ where $(x_0,\theta)$ were sampled from the uniform distribution over $\mathcal{X}_0 \times \Theta$ (Definition~\ref{def:random_rob_parameterized}).  Hence, $R$ has a probabilistic lower bound, $-a \in \mathbb{R}$, by definition of the robustness metric $\rho$ in Definition~\ref{def:robustness}.  Therefore, we can directly apply Corollaries~\ref{corr:ub_cvar} and~\ref{corr:ub_evar} to solve Problem~\ref{prob:2} as $\prob_{\pi_R}[r~|~-r \leq a] = 1$.  To formally state our results in this vein, we will redefine the loss function $L$ and upper bounding function $u_b$ for each risk measure in this specific application.  Here, $a$ is the upper bound for $-R$, \textit{i.e.} $\prob_{\pi_R}[r~|~-r \leq a] = 1$:
\begin{align}
    \label{eq:risk_verify_cvar_functions}
    \tag{CVaR}
    \hspace{-0.4 in} L(x,\mu,t) & = t\left( \mu + \frac{1}{\alpha}\max\left\{\frac{x}{t} - \mu, 0\right\}\right),  & & \hspace{-0.5 in}u_b(\mu,t) = L(a,\mu,t), \\
    \label{eq:risk_verify_evar_functions}
    \tag{EVaR}
    \hspace{-0.4 in} L(x,\mu,t) & = t\left(\mu + e^{\frac{x}{t} - \mu - \ln(\alpha) - 1} \right), & & \hspace{-0.5 in} u_b(\mu,t) = L(a,\mu,t).
\end{align}
Then our corollaries will follow.  First, we have that we can upper bound $\cvar_{\alpha}(-R)~\forall~\alpha \in (0,1]$.
\begin{corollary}
\label{corr:ub_cvar_verification}
Let $R$ be a scalar random variable as per Definition~\ref{def:randomized_robustness}, let $\alpha \in (0,1]$, let $L,u_b$ be as defined in equation~\eqref{eq:risk_verify_cvar_functions} with respect to this $\alpha$, and let $\zeta^*_N(\mu,t)$ be the solution to~\eqref{eq:upper_bound_scenario} for a set of $N$-samples $\{y_k = L(-r_k,\mu,t)\}_{k=1}^N$ of the random variable $Y=L(-R,\mu,t)$. Then, $\forall~\epsilon \in [0,1]$,
\begin{equation}
    \prob^N_{\pi_R}\left[r^*_C \triangleq \inf_{\mu \in \mathbb{R},~t>0}~\zeta^*_N(\mu,t)(1-\epsilon) + u_b(\mu,t)\epsilon \geq \cvar_{\alpha}(-R)\right] \geq 1-(1-\epsilon)^N.
\end{equation}
\end{corollary}
\begin{proof}
This is an application of Corollary~\ref{corr:ub_cvar} with the specific loss function $L$ and upper bounding function $u_b$ provided in equation~\eqref{eq:risk_verify_cvar_functions}.  The candidate upper bound for the random variable $-R$ is $a \in \mathbb{R}$ and stems via Definitions~\ref{def:robustness} and~\ref{def:randomized_robustness} insofar as samples $-r$ of $-R$ are the outputs of the robustness of a randomly sampled system trajectory, \textit{i.e.} $-r = -\rho(x^\theta)$ where $(x_0,\theta)$ was sampled uniformly from $\mathcal{X}_0 \times \Theta$.  Hence, $\prob_{\pi_R}[r~|~-r \leq a] = 1$.
\end{proof}

\noindent Then, we have that we can upper bound $\evar_{\alpha}(-R)~\forall~\alpha \in (0,1]$.
\begin{corollary}
\label{corr:ub_evar_verification}
Let $R$ be a scalar random variable as per Definition~\ref{def:randomized_robustness}, let $\alpha \in (0,1]$, let $L,u_b$ be as defined in equation~\eqref{eq:risk_verify_evar_functions} with respect to this $\alpha$, and let $\zeta^*_N(\mu,t)$ be the solution to~\eqref{eq:upper_bound_scenario} for a set of $N$-samples $\{y_k = L(-r_k,\mu,t)\}_{k=1}^N$ of the random variable $Y=L(-R,\mu,t)$. Then, $\forall~\epsilon \in [0,1]$,
\begin{equation}
    \prob^N_{\pi_R}\left[r^*_E \triangleq \inf_{\mu \in \mathbb{R},~t>0}~\zeta^*_N(\mu,t)(1-\epsilon) + u_b(\mu,t)\epsilon \geq \evar_{\alpha}(-R)\right] \geq 1-(1-\epsilon)^N.
\end{equation}
\end{corollary}
\begin{proof}
This is an application of Corollary~\ref{corr:ub_evar} and follows in the footsteps of the proof for Corollary~\ref{corr:ub_cvar_verification}.
\end{proof}

\noindent Finally, the fundamental sample requirement to generate both of these bounds arises through Corollary~\ref{corr:fund_sample_requirement} and will be formally stated.

\begin{corollary}
\label{corr:min_sample_cvar_verification}
Let $R$ be a scalar random variable as per Definition~\ref{def:randomized_robustness}, let $\epsilon \in (0,1)$, let $\gamma \in [0,1)$, let $\alpha \in (0,1]$, let $L,u_b$ be as defined in equation~\eqref{eq:risk_verify_cvar_functions} with respect to this $\alpha$, and let $\zeta^*_N(\mu,t)$ be the solution to~\eqref{eq:upper_bound_scenario} for a set of $N$-samples $\{y_k = L(-r_k,\mu,t)\}_{k=1}^N$ of the random variable $Y=L(-R,\mu,t)$. If $N \geq \frac{\log(1-\gamma)}{\log(1-\epsilon)}$, then
\begin{equation}
    \prob^N_{\pi_R}\left[r^*_C \triangleq \inf_{\mu \in \mathbb{R},~t>0}~\zeta^*_N(\mu,t)(1-\epsilon) + u_b(\mu,t)\epsilon \geq \cvar_{\alpha}(-R)\right] \geq \gamma.
\end{equation}
\end{corollary}
\begin{proof}
This is an application of Corollary~\ref{corr:fund_sample_requirement}.
\end{proof}
\begin{corollary}
\label{corr:min_sample_evar_verification}
Let $R$ be a scalar random variable as per Definition~\ref{def:randomized_robustness}, let $\epsilon \in (0,1)$, let $\gamma \in [0,1)$, let $\alpha \in (0,1]$, let $L,u_b$ be as defined in equation~\eqref{eq:risk_verify_evar_functions} with respect to this $\alpha$, and let $\zeta^*_N(\mu,t)$ be the solution to~\eqref{eq:upper_bound_scenario} for a set of $N$-samples $\{y_k = L(-r_k,\mu,t)\}_{k=1}^N$ of the random variable $Y=L(-R,\mu,t)$. If $N \geq \frac{\log(1-\gamma)}{\log(1-\epsilon)}$, then
\begin{equation}
    \prob^N_{\pi_R}\left[r^*_E \triangleq \inf_{\mu \in \mathbb{R},~t>0}~\zeta^*_N(\mu,t)(1-\epsilon) + u_b(\mu,t)\epsilon \geq \evar_{\alpha}(-R)\right] \geq \gamma.
\end{equation}
\end{corollary}
\begin{proof}
This is an application of Corollary~\ref{corr:fund_sample_requirement}.
\end{proof}

\subsection{Examples}
\label{sec:verification_examples}
To showcase the results of Corollary~\ref{corr:ub_cvar_verification} and~\ref{corr:ub_evar_verification}, we will verify a cooperative multi-agent robotic system using the fundamental sample requirements offered by Corollaries~\ref{corr:min_sample_cvar_verification} and~\ref{corr:min_sample_evar_verification}.  Specifically, we will identify lower bounds on expected worst-case system performance - \textit{i.e.} upper bounds on both $\cvar_{0.1}(-R)$ and $\evar_{0.1}(-R)$ - for a multi-agent system that is to avoid self-collisions while each agent reaches their respective goal.  As a case study, we will use the Georgia Tech Robotarium, wherein all the robots can be modeled as unicycles~\cite{robotarium}:
\begin{equation}
\label{eq:setting}
\begin{gathered}
    x = \begin{bmatrix}
    x,\\
    y, \\
    \theta
    \end{bmatrix},~\dot x = 
    \begin{bmatrix}
    v \cos(\theta), \\
    v \sin(\theta), \\
    \omega,
    \end{bmatrix},~u = [v, \omega]^T, \\
    \mathcal{X} = [-1,1] \times [-0.6,0.6] \times [0,2 \pi],~M = [I_2,~\mathbf{0}_{2x1}].
\end{gathered}
\end{equation}
As mentioned in~\cite{robotarium}, a Lyapunov-based controller drives each agent $x^i$ to their desired orientation $x^i_d \in \mathcal{X}$, and this control input is filtered in a barrier-based quadratic program to ensure that robots do not collide when multiple robots are moving simultaneously~\cite{ames2016control}.  As a result, this barrier-based filter provides a natural robustness measure as per Definition~\ref{def:robustness}.  Let the state vector for all $3$ robots $\mathbf{x}^T=[x^{1T},x^{2T},x^{3T}]$.  Then, we can generate two, key functions.  One function, $h_g$, that the system hopes to keep positive, and another function, $h_f$, that the system hopes to make positive over its trajectory.  Here, $x^i,x^i_d$ are the states and desired poses for robot $i$:
\begin{align}
    \label{eq:avoid}
    h_g(\mathbf{x}) & = \min_{i \neq j,~i,j \in [1,2,3]}~\|M(x^i - x^j)\| - 0.15. \\
    \label{eq:future}
    h_f(\mathbf{x}) & = \max_{i \in [1,2,3]}~0.1 - \|M(x^i - x^i_d)\|.
\end{align}
Intuitively, $h_g(\mathbf{x}) \geq 0$ implies that the robots are sufficiently far apart, and $h_f(\mathbf{x}) \geq 0$ means that all robots are sufficiently close to their goal.  From these two functions, we construct our robustness measure $\rho$ for a state-signal $\mathbf{x}^{\theta}$ where $\theta^T = [Mx_d^{1T},Mx_d^{2T},Mx_d^{3T}]$:
\begin{gather}
    \rho_g(\mathbf{x}^{\theta}) = \min\limits_{t \in [0,30]}~h_g\left( \mathbf{x}^{\theta}_t\right), \quad \rho_f(\mathbf{x}^{\theta}) = \max\limits_{t \in [0,30]}~h_f\left(\mathbf{x}^{\theta}_t\right), \\
    \label{eq:robustness}
    \rho(\mathbf{x}^{\theta}) = 
    \begin{cases}
    \rho_g(\mathbf{x}^{\theta}), & \mbox{if~} \rho_g(\mathbf{x}^{\theta}), \rho_f(\mathbf{x}^{\theta}) \geq 0, \\
    \max\left\{\rho_g(\mathbf{x}^{\theta}), -0.1 \right\}, & \mbox{if~} \rho_g(\mathbf{x}^{\theta}) < 0, \\
    \max\left\{\rho_f(\mathbf{x}^{\theta}), -0.1\right\}, & \mbox{else}.
    \end{cases}
\end{gather}

\begin{figure}[t]
    \centering
    \includegraphics{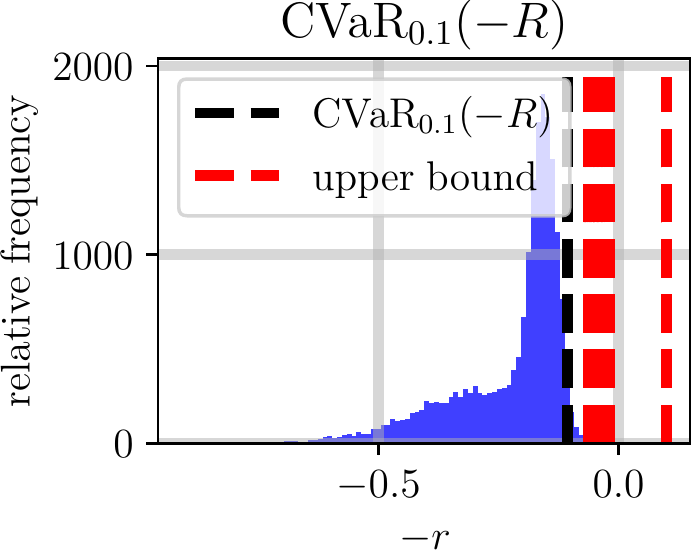}
    \includegraphics{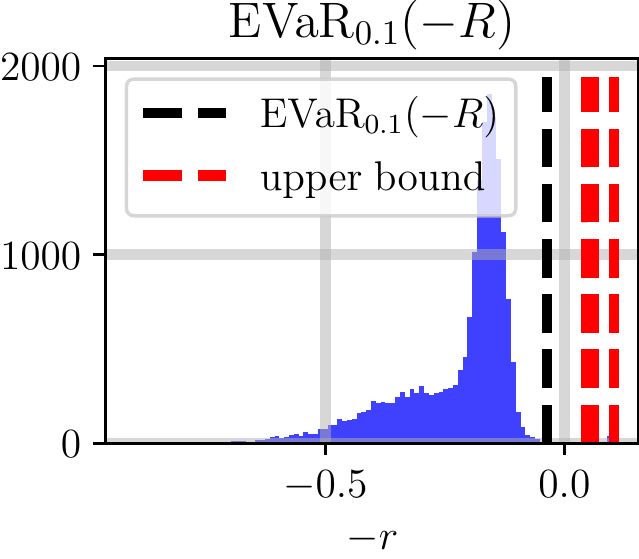}
    \caption{As an application of the results in Section~\ref{sec:concentration_inequalities}, we upper bound the risk measures of a multi-agent robotic system when its state trajectory is evaluated through a robustness metric (Definition~\ref{def:robustness}).  Shown above is this upper bounding procedure performed in Section~\ref{sec:verification_examples} for both $\cvar$ (left) and $\evar$ (right).  For each of $50$ trials, the upper bounds (red) for both risk measures are indeed greater than or equal to their ``true" counterparts (black).  These ``true" counterparts were calculated by taking $20000$ samples of the randomized system robustness $R$ (Definition~\ref{def:randomized_robustness}), and the distribution of samples is shown (blue).  The fact that the upper bounds are indeed upper bounds over all trials serves as numerical confirmation of Corollaries~\ref{corr:min_sample_cvar_verification} and~\ref{corr:min_sample_evar_verification} and their parent Corollary~\ref{corr:fund_sample_requirement}.  They also support the repeatability of our procedure in identifying upper bounds to $g$-entropic risk measures with high probability.}
    \label{fig:verification_data}
\end{figure}

If $\rho(\mathbf{x}^\theta) \geq 0$ then, all $3$ robots stayed at least $0.15$ meters from each other for $30$ seconds - as $h_g(\mathbf{x}^\theta_t) \geq 0,~\forall~t \in [0,30]$ - and reached within $0.1$ meters of their goal within $30$ seconds - as $h_f(\mathbf{x}^\theta_t) \geq 0,$ for some $t \in [0,30]$.  Furthermore, we also know this robustness measure $\rho$ outputs values that are always greater than or equal to $-0.1$.  This robustness measure $\rho$ also lets us specifically define our holistic system robustness $R$ in this case.  For this example, samples $r$ of $R$ arise when we first uniformly randomly sample initial locations and parameters $(\mathbf{x}_0,\theta)$ from their respective spaces below:
\begin{align}
    \mathcal{X}_0  = \{\mathbf{x} \in \mathcal{X}^{3}~|~h_g(\mathbf{x}) \geq 0.3 \}, \quad \Theta = \{P\mathbf{x} \in \mathcal{X}^{3}~|~h_g(\mathbf{x}) \geq 0.3 \}.
\end{align}
Then, we record the corresponding state trajectory of the multi-agent system $\mathbf{x}^{\theta}$ for at-least $30$ seconds.  The holistic system robustness sample $r = \rho(\mathbf{x}^\theta)$ then, with $\rho$ as defined in~\eqref{eq:robustness}.

\spacing
\newidea{Verifying a $3$ Robot System:} Our goal is to determine an upper bound on both the $\cvar_{\alpha}(-R)$ and $\evar_{\alpha}(-R)$ with $\alpha = 0.1$ for the three robots as they carry out their task.  Furthermore, we hope to determine this upper bound with $95\%$ confidence, \textit{i.e.} $\gamma = 0.95$, and we require that $\epsilon = 0.02$.  For this case, Corollaries~\ref{corr:min_sample_cvar_verification} and~\ref{corr:min_sample_evar_verification} require that we take at minimum $N\geq 149$ samples of our holistic system robustness $R$ to determine this upper bound.  As a result, we uniformly sampled $N = 149$ initial conditions and parameters $(\mathbf{x}_0,\theta)$ from $\mathcal{X}_0\times \Theta$, recorded the corresponding state trajectories $\mathbf{x}^{\theta}$, and recorded their robustnesses $r = \rho(\mathbf{x}^{\theta})$.  We also repeated this data collection procedure $50$ times to ensure that our results are repeatable, and to identify the true $\cvar_{\alpha}(-R)$ and $\evar_{\alpha}(-R)$ we repeated this data collection procedure once more, but took $N=20000$ samples instead.  If Corollaries~\ref{corr:min_sample_cvar_verification} and~\ref{corr:min_sample_evar_verification} are correct, then with $95\%$ confidence, we expect that our sampled upper bounds for the $50$ trials performed are greater than or equal to their true counterparts.  All this data is shown in Figure~\ref{fig:verification_data}, where as prior, the black lines indicate the true risk measure evaluation and all red lines are the upper bounds generated per trial.  As can be seen, all generated upper bounds are indeed upper bounds for their counterparts which serves as numerical confirmation of Corollaries~\ref{corr:min_sample_cvar_verification} and~\ref{corr:min_sample_evar_verification} and also of the repeatability of our procedure in identifying upper bounds to these risk measures with high-confidence.  

\newidea{Brief Remark on Relative Bound Tightness:} Figure~\ref{fig:verification_data} does prompt one question, however.  Specifically, it appears the upper bounds for $\cvar$ are tighter than those for $\evar$.  We anticipate this might arise as $\cvar$ is a less conservative risk measure than $\evar$.  As a result, for a given confidence level $\gamma$ in our upper bound, the upper bound for $\evar$ will likely be looser on account of this conservatism. 

\section{Identification of ``Good" Solutions to Optimization Problems}
\label{sec:decision_selection}
Risk-aware verification~\cite{akella2022scenario} is one example of a type of optimization problem that could potentially benefit from the sample-based probabilistic approaches developed in Section~\ref{sec:concentration_inequalities}.  Another optimization problem is risk-aware policy synthesis.  However, such an optimization problem will likely be non-convex, and therefore, will be difficult to solve exactly.  That being said, if the goal is the identification of a ``good" policy - \textit{i.e.} a policy that is better than $90\%$ of all other policies we could have generated - this problem is significantly easier.  This section details our efforts in that vein, as we develop the general, mathematical basis that will be utilized in Section~\ref{sec:policy_synthesis}.

\subsection{Notation and Setting}
To simplify the setting, we assume we have a reward function $R: \mathbb{X} \subseteq \mathbb{R}^n \to \mathbb{R}$ for some positive integer $n$ and general domain $\mathbb{X}$.  Then, the ideal optimization problem to-be-solved is
\begin{equation}
    \label{eq:general_opt_problem}
    \max_{x \in D \subseteq \mathbb{X}}~R(x).
\end{equation}
For~\eqref{eq:general_opt_problem}, the decision space $D$ and reward function $R$ satisfy the following assumption:
\begin{assumption}
\label{assump:bounded_decisions}
The decision space $D$ has bounded volume, \textit{i.e.} $\int_{D}~1~dx = M < \infty$, and the reward function $R$ attains its maximum value over $D$, \textit{i.e.} $\exists~x^* \in D$ such that $R(x^*) = R^* < \infty$ and $R(x^*) \geq R(x),~\forall~x \in D$.
\end{assumption}
\noindent This assumption is not too restrictive, as a large class of optimization problems fall under this setting, including risk-aware policy synthesis.

Additionally, we define a function $\volfrac: 2^D \to \mathbb{R}$ that identifies the volume fraction that a subspace $A$ of $D$ occupies:
\begin{equation}
    \label{eq:volume_fraction}
    \volfrac(A) = \frac{\int_A~1~dx}{\int_D~1~dx}.
\end{equation}
Finally, let $F: D \to 2^D$ be a function that identifies the space of decisions $x' \in D$ that are ``better" than the provided decision $x$, \textit{i.e.}
\begin{equation}
    \label{eq:falsifying_set}
    F(x) = \{x' \in D~|~R(x') > R(x)\}.
\end{equation}
Then this section considers the following problem.
\begin{problem}
\label{prob:decision_selection}
For any $\epsilon \in (0,1)$ devise a method to find a decision $x \in D$ such that $x$ is at least in the $100(1-\epsilon)$-th percentile of all possible decisions $x \in D$ with respect to the reward function $R$, \textit{i.e.} find a decision $x$ such that $\volfrac(F(x)) \leq \epsilon$, with $\volfrac$ defined in equation~\eqref{eq:volume_fraction} and $F$ defined in equation~\eqref{eq:falsifying_set}.
\end{problem}

\subsection{Sampling Methods for Identification of ``Good" Decisions}
First, we note that for a uniform distribution over the decision space $D$, the probability of sampling a decision $x \in A \subset D$ is equivalent to the volume fraction of $A$ with respect to $D$.
\begin{lemma}
\label{lem:prob2volume}
Let Assumption~\ref{assump:bounded_decisions} hold.  The following statement is true with $\volfrac$ as defined in~\eqref{eq:volume_fraction}:
\begin{equation}
    \prob_{\uniform[D]}[x~|~x \in A \subseteq D] = \volfrac(A).
\end{equation}
\end{lemma}
\begin{proof}
As the distribution is uniform, we have the following chain of equalities:
\begin{equation}
    \prob_{\uniform[D]}[x~|~x \in A] = \frac{\int_A~1~dx}{\int_D~1~dx} = \volfrac(A).
\end{equation}
\end{proof}

Second, we note that~\eqref{eq:general_opt_problem} can be recast as an uncertain program similar to~\eqref{eq:uncertain_program}.  Specifically, let $X$ be a random variable whose distribution is the uniform distribution over $D$, \textit{i.e.} $\uniform[D]$.  Then, the corresponding random reward is $Y = R(X)$.  This random reward variable $Y$ has its own samples $y$ and (unknown) distribution $\pi_Y$, letting us construct an uncertain program mirroring~\eqref{eq:general_opt_problem}:
\begin{equation}
    \label{eq:recasting_general_opt}
    \tag{UP-G}
    \begin{aligned}
        \zeta^* & = \argmin_{\zeta \in \mathbb{R}}~ & & \zeta, \\
        &~~\mathrm{subject~to}~ & &\zeta \geq y,~Y \sim \pi_Y~\mathrm{with~samples~}y.
    \end{aligned}
\end{equation}
This uncertain program~\eqref{eq:recasting_general_opt} also has an analagous scenario program:
\begin{align}
        \zeta^*_N & = \argmin_{\zeta \in \mathbb{R}}~ & & \zeta, \label{eq:scenario_general_opt}
        \tag{RP-G}\\
        &~~\mathrm{subject~to}~ & &\zeta \geq y_i,~y_i \in \{y_k = R(x_k)\}_{k=1}^N,~X\sim \uniform[D]~\mathrm{with~samples~x}.
\end{align}

This scenario program is crucial to our approach for two reasons.  First, there exists a sampled decision $x_i$ in the set of all sampled decisions $\{x_k\}_{k=1}^N$ such that the reward of this decision is equivalent to the scenario solution $\zeta^*_N$, \textit{i.e.} $R(x_i) = \zeta^*_N$.  Second, via Theorem~\ref{thm:scenario_opt} we can upper bound the probability of sampling another decision $x' \in D$ such that $R(x') > \zeta^*_N = R(x_i)$.  Via Lemma~\ref{lem:prob2volume}, such a probability corresponds to the volume fraction of those decisions $x' \in D$ that are ``better" than $x_i$, \textit{i.e.} this probability corresponds to $\volfrac(F(x_i))$.  As a result, in order to determine a decision $x \in D$ that is ``better" than a predetermined volume fraction $\epsilon \in (0,1)$ of all possible decisions $x'\in D$ with confidence $\gamma \in [0,1)$, we just need to take enough samples of $X$ and solve~\eqref{eq:scenario_general_opt}.  This idea is formalized in the following theorem:
\begin{theorem}
\label{thm:good_solutions}
Let $\epsilon \in (0,1)$, let $\gamma \in [0,1)$, let Assumption~\ref{assump:bounded_decisions} hold, let $\zeta^*_N$ be the solution to~\eqref{eq:scenario_general_opt} for an $N$-sample set $\{x_k\}_{k=1}^N$ of $X$, let $\volfrac$ be as defined in~\eqref{eq:volume_fraction}, and let $F$ be as defined in~\eqref{eq:falsifying_set}.  If $N \geq \frac{\log(1-\gamma)}{\log(1-\epsilon)}$ then with minimum probability $\gamma$, there exists at least one sampled decision $x_i \in \{x_k\}_{k=1}^N$ that is at-least in the $100(1-\epsilon)$-th percentile of all possible decisions $x \in D$, \textit{i.e.},
\begin{equation}
    \exists~x_i \in \{x_k\}_{k=1}^N \suchthat \prob^N_{\uniform[D]}[\volfrac(F(x_i)) \leq \epsilon] \geq \gamma, \mathrm{~and~} \zeta^*_N = R(x_i).
\end{equation}
\end{theorem}
\begin{proof}
First, for any finite sample set $\{x_k\}_{k=1}^N$ there exists a decision $x_i \in \{x_k\}_{k=1}^N$ such that $\zeta^*_N = R(x_i)$.  This is due to the fact that~\eqref{eq:scenario_general_opt} is a linear program minimizing a scalar decision variable subject to a series of lower bounds that take values in $\mathbb{R}$.  As a result, $\zeta^*_N$ must equal one of its lower bounds to be a valid solution to~\eqref{eq:scenario_general_opt}, and as such, it must be equal to $R(x_i)$ for at least one $x_i \in \{x_k\}_{k=1}^N$.

Then via Theorem~\ref{thm:scenario_opt}, the following probabilistic inequality holds for our scenario solution $\zeta^*_N$, where the violation probability $V(\zeta^*_N)$ has been replaced with its appropriate definition in this context.
\begin{equation}
    \label{eq:violation_extrapolation}
    \prob^N_{\uniform[D]}[\prob_{\uniform[D]}[x~|~R(x) > \zeta^*_N] \leq \epsilon] \geq 1-(1-\epsilon)^N.
\end{equation}
By definition of $F$ in~\eqref{eq:falsifying_set},~\eqref{eq:violation_extrapolation} can be recast as follows for some $x_i \in \{x_k\}_{k=1}^N$:
\begin{equation}
    \prob^N_{\uniform[D]}[\prob_{\uniform[D]}[x~|~x \in F(x_i)] \leq \epsilon] \geq 1-(1-\epsilon)^N.
\end{equation}
Via Lemma~\ref{lem:prob2volume} we can replace the inner probability with the volume fraction of $F(x_i)$:
\begin{equation}
    \prob^N_{\uniform[D]}[\volfrac(F(x_i)) \leq \epsilon] \geq 1-(1-\epsilon)^N.
\end{equation}
Then, as $N\geq \frac{\log(1-\gamma)}{\log(1-\epsilon)}$, $\epsilon \in (0,1)$ and $\gamma \in [0,1)$, $1-(1-\epsilon)^N \geq \gamma$, and we have our result.
\end{proof}

To summarize, Theorem~\ref{thm:good_solutions} states that if we want to find a ``good" solution to~\eqref{eq:general_opt_problem} with minimum probability $\gamma$ - ``good" in the sense that it is in the $100(1-\epsilon)$-th percentile of all decisions $x \in D$ that we could take - that we are required to evaluate at minimum $N \geq \frac{\log(1-\gamma)}{\log(1-\epsilon)}$ uniformly randomly chosen decisions $x \in D$.  This is due to the fact that if we evaluate the reward $R$ for each such sampled decision in our sample set $\{x_k\}_{k=1}^N$, at least one decision $x_i \in \{x_k\}_{k=1}^N$ is guaranteed, with minimum probability $\gamma$, to produce a reward $R(x_i)$ that is in the $100(1-\epsilon)-$th percentile of all possible rewards achievable.  Hence, this decision $x_i$ that produces this reward $R(x_i)$ is also guaranteed, with minimum probability $\gamma$, to be in the $100(1-\epsilon)$-th percentile of all possible decisions $x \in D$ with respect to the reward $R$.

\subsection{Examples}
\label{sec:decision_examples}
This section, provides a few examples of applying Theorem~\ref{thm:good_solutions} to finding paths for the traveling salesman problem (TSP) that are in the $99$-th percentile of all possible paths achievable~\cite{flood1956traveling}.  Specifically, TSP references the identification of a permutation of integers $P = [i_1, i_3, i_6, \dots]$ where such a permutation minimizes the summation of a function $d: \mathbb{Z}_+ \times \mathbb{Z}_+ \to \mathbb{R}$.  That is, if we define a finite set of integers $I \subset \mathbb{Z}_+$ , we can also define a set $P \subset 2^I$ of all permutations of integers in $I$, \textit{i.e.} $P = \{p = (i_1, \dots, i_{|I|})~|~\forall~i_j \in I,~i_j \in p$ and $i_j$ appears only once in $p\}$.  Then, the canonical traveling salesman optimization problem is
\begin{equation}
    \label{eq:traveling_salesman}
    \min_{s \in P}~\sum_{k = 1}^{|s|-1} d(p_k,p_{k+1}) + d(p_{|I|},p_1).
\end{equation}
While phrased as a minimization,~\eqref{eq:traveling_salesman} can be recast as a maximization problem of the form in~\eqref{eq:general_opt_problem} - the reward function $R$ is the negation of the objective function above, and the decision space $D = P$.  Furthermore, this problem satisfies Assumption~\ref{assump:bounded_decisions}, with $\int_{P}~1~dx = |P| < \infty$, which means that we should be able to follow Theorem~\ref{thm:good_solutions} to identify a ``good" solution to this problem.

Specifically then, we will randomly generate $9$ nodes $n_i \in [-5,5]^2 \subset \mathbb{R}^2$. Our cost function $d(i,j) = \|n_i - n_j\|$.  As we have generated $9$ nodes, the cardinality of the set of all possible permutations $|P| = 9! = 362880$.  However, according to Theorem~\ref{thm:good_solutions}, if we want a permutation $p \in P$ that is in the $99$-th percentile with $95\%$ confidence, then we need to uniformly sample $N \geq \frac{\log(1-\gamma = 0.95)}{\log(1-\epsilon = 0.01)} \approx 299$ permutations $p \in P$ and evaluate their corresponding costs.  Doing this procedure once, we identified a permutation $p^* \in P$ that was in the $99.9997$-th percentile of all paths as shown on the left in Figure~\ref{fig:salesman_information}.  This one case shows the utility in our sample approach in providing ``good" solutions to difficult optimization problems insofar as we only had to evaluate $299/362880 \approx 0.0008$ of all possible paths to generate a very good one.  Indeed, the method also repeatably identifies ``good" solutions as shown on the right in Figure~\ref{fig:salesman_information}.  In this case, we required a path in the $95$-th percentile with minimum probability $1-10^{-6}$ for a case where we randomly generated $8$ nodes $n_i \in [-5,5]^2 \subset \mathbb{R}^2$.  In every case, the identified path $p^*$ was in the $95$-th percentile as evidenced by $\volfrac(F(p^*)) \leq \epsilon$ which is guaranteed via Theorem~\ref{thm:good_solutions}.

\begin{figure}[t]
    \centering
    \includegraphics{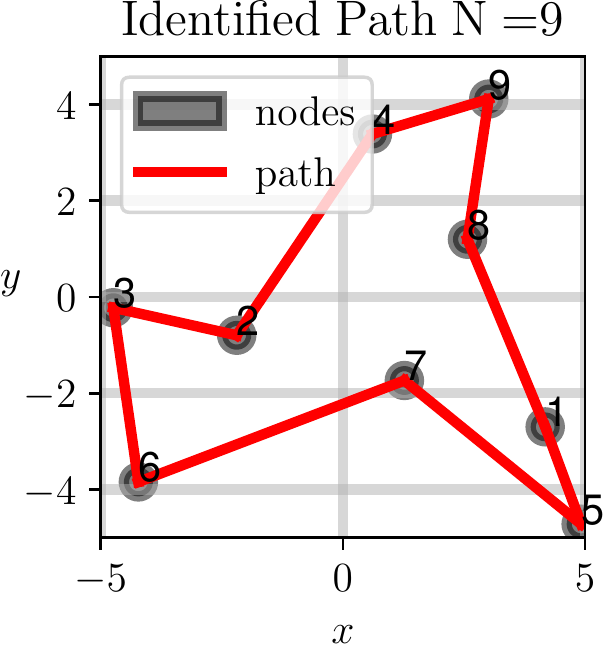}
    \hspace{0.2 in}
    \includegraphics{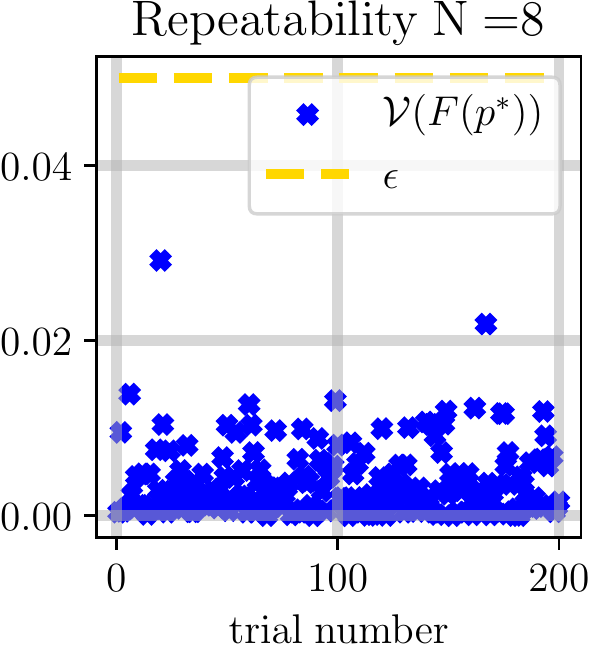}
    \caption{Risk-Aware Policy Synthesis (Section~\ref{sec:policy_synthesis}) requires a method to identify good solutions to potentially non-convex optimization problems - this is the procedure detailed in Section~\ref{sec:decision_selection}.  Shown above is an application of this procedure to identify ``good" paths for a traveling salesman problem~\cite{flood1956traveling}.  By uniformly sampling $N=299$ paths from the set of all possible paths $P$ with $|P| = 362880$, we are able to identify (left) a path that is in the $99.9997$-th percentile of all paths.  This procedure also repeatably identifies ``good" paths, \textit{a.k.a} ``good" decisions,  as shown on the figure on the right.  If we take the minimum number of samples offered by Theorem~\ref{thm:good_solutions} to identify a path that in the $95$-th percentile with minimum probability $1-10^{-6}$, we see that over $200$ trials - taking $N = 270$ samples each time - all determined paths are in the $95$-th percentile as $\volfrac(F(p^*)) \leq \epsilon = 0.05$.}
    \label{fig:salesman_information}
\end{figure}

\section{Risk-Aware Policy Synthesis}
\label{sec:policy_synthesis}
\subsection{Setting and a Clarifying Example}
Another type of optimization problem that falls under the general class of problems studied in Section~\ref{sec:decision_selection} is risk-aware policy synthesis.  While general policy synthesis, especially in the model-free RL setting, attempts to design a policy without a prior, this section designs a policy that selects actions against an existing control form.  The goal of policy synthesis then would be to identify parameters of this controller such that the resulting controller performs admirably, in a risk-aware sense, at satisfying an objective of interest.  Similar works in this vein attempt to address this problem via Bayesian parameter selection~\cite{shalloo2020automation, calandra2016bayesian, berkenkamp2021bayesian} among other techniques.

Satisfactory policy synthesis can be seen as a specific application of Sections~\ref{sec:verification} and~\ref{sec:decision_selection}, and as such, we will reference the same class of systems we considered in Section~\ref{sec:verification_examples} as described in~\eqref{eq:evolution}, with one addition.  Specifically, we still focus on a general class of systems with the state space, $\mathcal{X}$; the input space, $\mathcal{U}$; a space of problem parameters $\Theta$ that the controller $U$ must account for.  The discrepancy with~\eqref{eq:evolution} arises in the parameterization of the controller $U$ via the parameter space $P$ of controller parameters.  As before, we assume the system is subject to some stochastic noise $\xi$ with an unknown distribution $\pi_{\xi}(x,u,t)$ over $\mathbb{R}^n$.
\begin{align}
    \dot x & = f(x,u) + \xi, & & x \in \mathcal{X} \subset \mathbb{R}^n,~u \in \mathcal{U} \subset \mathbb{R}^m,~\label{eq:evolution_synthesis}\\
    u & = U(x,\theta,p), & & \theta \in \Theta \subset \mathbb{R}^l,~p \in P \subset \mathbb{R}^s,~\label{eq:controller_param}\\
    \xi & \sim \pi_{\xi}(x,u,t), & & \int_{\mathcal{X}}\pi_{\xi}(x,u,t,s)~ds = 1~\forall~x,u,t.
\end{align}
As before, we will define $x^{\theta,p}_t$ as the solution to our closed-loop system at time $t$, given the initial condition $x_0 \in \mathcal{X}_0 \subseteq \mathcal{X}$, an accounting parameter $\theta \in \Theta$, and a set of controller parameters $p \in P$.
\begin{equation}
    \label{eq:full_params_CL_sys}
    \dot x^{\theta,p}_t = f\left(x^{\theta,p}_t, U\left(x^{\theta,p}_t, \theta,p\right)\right) + \xi,~x^{\theta,p} \in \signalspace,~\signalspace = \{s:\mathbb{R}_{\geq 0} \to \mathbb{R}^n\}.
\end{equation}

The following example will be used to clarify the setting.  Consider a simple planar robot modelable via unicycle dynamics as analyzed before in Section~\ref{sec:verification_examples}.  Let's assume that its objective is to reach a goal, whose center location is parameterized by a vector $g \in [-2,2]^2 \subset \mathbb{R}^2$.  Furthermore, let's assume that several other agents are starting in random locations $o_i \in [-2,2]^2$, that are moving about randomly, and that our robot must avoid.  The full set of accounting parameters $\theta^T = [g^T, o_1^T, \dots, o^T_M]$.  Let's further assume that this robot's path-planner follows a predator-prey model, and as such is parameterized via two parameters $\alpha,\beta \in \mathbb{R}$ that correspond to its aggressiveness/hesitancy of moving towards the goal/other agents respectively.  The controller parameters $p = [\alpha,\beta]$.  As also expressed in Section~\ref{sec:verification}, we further assume the existence of a robustness metric $\rho$ as per Definition~\ref{def:robustness} that maps the agent's trajectory $x^{\theta,p}$ to a positive number whenever the robot successfully reaches its goal within a pre-specified amount of time and also avoids all other agents.

\subsection{Problem Statement}
Our solution to this problem builds off the risk-aware verification work in Section~\ref{sec:verification} and the ``good" decision selection work in Section~\ref{sec:decision_selection}.  Specifically, choosing a parameter $p \in P$ defines a specific closed-loop system - that system for which the controller $U$ is parameterized via this parameter $p$.  Then for this system, we can calculate an upper bound on the Conditional-Value-at-Risk of the system's robustness $\rho$ over the entire space of initial conditions and accounting parameters $\mathcal{X}_0 \times \Theta$ - this is Corollary~\ref{corr:ub_cvar_verification}.  We could follow a similar process for upper bounding the Entropic-Value-at-Risk as well.  However, we will focus on $\cvar$ for brevity.  This does raise the question: \textit{can we minimize this upper bound over the set of possible parameters $P$}?  This is the risk-aware policy synthesis question.  To phrase this question formally, we have to modify Definition~\ref{def:randomized_robustness} for this setting.
\begin{definition}
\label{def:random_rob_parameterized} The \textit{parameterized holistic system robustness},
$R_p$, is a scalar random variable with distribution $\pi_{R_p}$ and samples $r_p$ denoting the closed-loop robustness $\rho(x^{\theta,p})$ of trajectories $x^{\theta,p}$ whose initial condition and parameter $(x_0,\theta)$ were sampled uniformly from $\mathcal{X}_0 \times \Theta$ and whose controller $U$ is parameterized via the controller parameter $p \in P$ as per equation~\eqref{eq:controller_param}
\end{definition}

Per Corollary~\ref{corr:ub_cvar_verification}, if we fix the parameter $p$ for the controller $U$, we can generate an upper bound $r^*_{C,p}$ on $\cvar_{\alpha}(-R_p)$ for any $\alpha \in (0,1]$ with high-confidence.  As we study only $\cvar$ in this section, we will drop the subscript $C$ for $r^*_{C,p}$ to simplify notation.  Mathematically, this lets us generate a mapping $\riskmap$ from controller parameter $p$, confidence $\gamma$, and risk-level $\alpha$, to the upper bound $r^*_p$ generated with confidence $\gamma$ for $\cvar_{\alpha}(-R_p)$. Implicit in this function definition will be the number of samples $N_{\riskmap}$ we take of the initial conditions and accounting parameters $(x_0,\theta)$ from the uniform distribution over $\mathcal{X}_0 \times \Theta$:
\begin{equation}
    \label{eq:riskmap}
    \riskmap(p,\gamma,\alpha) = r^*_p, \suchthat \prob^{N_{\riskmap}}_{\pi_{R_p}}\left[r^*_p \geq \cvar_{\alpha}(-R_p)\right] \geq \gamma.
\end{equation}
Then, by Definitions~\ref{def:var} and~\ref{def:cvar}, we have the following proposition indicating that the risk map $\riskmap$ serves as a high-confidence lower bound on the expected worst-case system performance.
\begin{proposition}
\label{prop:polsynth_motivation}
Let $\riskmap$ be as defined in equation~\eqref{eq:riskmap} for some $\gamma \in [0,1)$ and $\alpha \in (0,1]$.  The following probabilistic inequality holds, with $\pi_{R_p}$ the distribution of the parameterized holistic system robustness $R_p$ defined in Definition~\ref{def:random_rob_parameterized}:
\begin{equation}
    \label{eq:polsynth_motivating_lb}
    \prob^N_{\pi_{R_p}}\left[-\riskmap(p,\gamma,\alpha) \leq \expect_{\pi_{R_p}}[r~|~r \leq \var_{1-\alpha}(R_p)] \right] \geq \gamma.
\end{equation}
\end{proposition}
\begin{proof}
By definition of $\cvar_{\alpha}$ for any risk-level $\alpha \in (0,1]$ in Definition~\ref{def:cvar}, we have the following equivalency:
\begin{equation}
    \cvar_{\alpha}(-R_p) = \expect_{\pi_{R_p}}[r_p~|~-r_p \geq \var_{\alpha}(-R_p)].
\end{equation}
Then by definition of $\var$ and a symmetry argument when flipping the random variable of interest, we have the following equality:
\begin{equation}
    -\cvar_{\alpha}(-R_p) = \expect_{\pi_{r_p}}[r_p~|~r_p \leq \var_{1-\alpha}(R_p)].
\end{equation}
Then the desired inequality stems from the definition of the risk map $\riskmap$ in equation~\eqref{eq:riskmap}.
\end{proof}

As a result, the goal of policy synthesis should be to minimize $\riskmap(p,\gamma,\alpha)$.  As, for some fixed $\gamma,\alpha$, minimization of $\riskmap(p,\gamma,\alpha)$ over $P$ corresponds to the identification of a set of controller parameters $p \in P$ such that in the worst $100\alpha\%$ of cases, the expected minimum system robustness is maximized as stated in Proposition~\ref{prop:polsynth_motivation}.  Since the robustness measure $\rho$ outputs positive numbers for those trajectories that satisfy the desired behavior (Definition~\ref{def:robustness}), if this maximum lower bound were positive, then the expected performance in the worst $100\alpha\%$ of cases would still exhibit satisfactory behavior. This leads to the nominal, risk-aware synthesis problem:
\begin{equation}
    \label{eq:synth_general_opt}
    \min_{p \in P}~\riskmap(p,\gamma,\alpha),~\mathrm{for~some~}\gamma,\alpha \in [0,1),
\end{equation}
which is similar to the general problem~\eqref{eq:general_opt_problem} studied in Section~\ref{sec:decision_selection}.  The decision space $D = P$ and the reward function $R = -\riskmap(\gamma,\alpha)$.  We can also rewrite the volume fraction function $\volfrac$~\eqref{eq:volume_fraction} and ``better" policy map $F$~\eqref{eq:falsifying_set}:
\begin{equation}
    \label{eq:polsynth_specific_functions}
    \volfrac(A) = \frac{\int_A~1~dx}{\int_P~1~dx}, \quad F(p,\gamma,\alpha) = \left\{p' \in P~|~\riskmap(p',\gamma,\alpha) < \riskmap(p,\gamma,\alpha)\right\}.
\end{equation}
This lets us formally define the problem under study for this section.
\begin{problem}
\label{prob:risk_aware_polsynth}
Let $\riskmap$ be as defined in~\eqref{eq:riskmap} with respect to some $\alpha \in (0,1]$ and $\gamma \in (0,1)$.  Let $\epsilon \in (0,1)$ as well.  Find a set of controller parameters $p \in P$ such that the corresponding controller $U$ is at-least in the $100(1-\epsilon)$-th percentile of all possible controllers with respect to minimization of $\riskmap$, \textit{i.e.} find $p \in P$ such that $\volfrac(F(p)) \leq \epsilon$ where $\volfrac,F$ are defined in~\eqref{eq:polsynth_specific_functions}.
\end{problem}

\subsection{Identification of ``Good" Risk-Aware Policies}
To start, Problem~\ref{prob:risk_aware_polsynth} is a refinement of Problem~\ref{prob:decision_selection}.  As a result, identification of a ``good" policy only requires construction of a similar scenario program to~\eqref{eq:scenario_general_opt}:
\begin{align}
        \zeta^*_N & = \argmin_{\zeta \in \mathbb{R}}~ & & \zeta, \label{eq:scenario_polsynth}
        \tag{RP-PS}\\
        &~~\mathrm{subject~to}~ & &\zeta \geq y_i,~y_i \in \left\{y_k = \riskmap(p_k,\gamma,\alpha)\right\}_{k=1}^N, \\
        & & & \qquad \qquad p_k~\mathrm{are~drawn~uniformly~from~}P.
\end{align}
Then, identification of ``good" policies is a direct consequence of Theorem~\ref{thm:good_solutions}.  To facilitate the statement of this corollary and a theorem to follow, we will state a common assumption.
\begin{assumption}
\label{assump:risk_aware_polsynth}
$\riskmap$ is as defined in equation~\eqref{eq:riskmap} with respect to some $\alpha \in (0,1]$ and $\gamma_1 \in [0,1)$.  $\epsilon \in (0,1)$, $\gamma_2\in [0,1)$, $\zeta^*_N$ is the solution to~\eqref{eq:scenario_polsynth} for a set of $N$-samples $\{y_k = \riskmap(p_k,\gamma_1,\alpha)\}_{k=1}^N$ where each $p_k$ was drawn uniformly from $P$, and $\volfrac,F$ are as defined in equation~\eqref{eq:polsynth_specific_functions}.
\end{assumption}
\begin{corollary}
\label{corr:synth_fund_requirement}
Let Assumption~\ref{assump:risk_aware_polsynth} hold.  If $N \geq \frac{\log(1-\gamma_2)}{\log(1-\epsilon)}$, then with minimum probability $\gamma_2$ there exists at least one parameter set $p_i \in \{p_k\}_{k=1}^N$ whose associated controller $U$ is at-least in the $100(1-\epsilon)$-th percentile of controllers possible with respect to minimizing $\riskmap$, \textit{i.e.}
\begin{equation}
    \exists~p_i \in \{p_k\}_{k=1}^N, \suchthat \prob^N_{\uniform[P]}\left[\volfrac(F(p_i)) \leq \epsilon \right] \geq \gamma_2,~\mathrm{and}~\zeta^*_N = \riskmap(p_i,\gamma_1,\alpha).
\end{equation}
\end{corollary}
\begin{proof}
This is a direct application of Theorem~\ref{thm:good_solutions}.
\end{proof}

To summarize, for each uniformly sampled controller parameterization $p \in P$, we evaluate the corresponding controller $U$ against $N_{\riskmap}$ uniformly sampled initial conditions and accounting parameters $(x_0,\theta) \in \mathcal{X}_0 \times \Theta$ and evaluate the robustness $\rho(x^{\theta,p})$ of the corresponding trajectories.  By evaluating enough such trajectories, we generate - via Corollary~\ref{corr:ub_cvar_verification} and Proposition~\ref{prop:polsynth_motivation} - a lower bound $-r^*$ on the expected worst-case system performance in the worst $100\alpha\%$ of cases.  By repeating this procedure for enough samples - $N \geq \frac{\log(1-\gamma_2)}{\log(1-\epsilon)}$ samples - we are guaranteed with minimum probability $\gamma_2$ to find at least one parameter set $p_i$ in the sampled set $\{p_k\}_{k=1}^N$ such that the corresponding controller $U$ is at-least in the $100(1-\epsilon)$-th percentile of all possible controllers with respect to maximizing this lower bound.  This probabilistic guarantee is due to Theorem~\ref{thm:good_solutions}.  As we also provide a minimum sample requirement on the number of controllers required to be tested in this procedure, we state that we have a fundamental sample requirement to generate ``good" risk-aware policies with high confidence.  Furthermore, as this sample requirement is based on both our desired confidence $\gamma$ and desired percentile level $1-\epsilon$, we state that this procedure also identifies the relative complexity in the identification of a ``better" policy, insofar as it provides a sample requirement for doing so.

\subsection{Examples}
\label{sec:synthesis_examples}
In this section, we will synthesize a risk-aware controller for the cooperative multi-agent system whose controller was verified in Section~\ref{sec:verification_examples}, and we will show that the synthesized controller outperforms the baseline controller when accounting for worst-case system performance.  To reiterate, we will use the robotarium as a case study again, where the robots can be modeled as unicycles~\cite{robotarium}.  The dynamics are the same as in~\eqref{eq:setting} in Section~\ref{sec:verification_examples}, and will not be repeated to be brief.

Different from the verification case in Section~\ref{sec:verification_examples}, however, we will now assume that the hybrid Lyapunov-barrier controller the robots are equipped with~\cite{robotarium} is now to-be-optimized.  Since we only assume knowledge of the controller parameters $p$ and not the specific controller form $U$, we will refrain from specifying the controller in favor of stating that the controller parameters we can vary are the gains for the Lyapunov controller - approach angle gain, $p_1$; desired angle gain, $p_2$; rotation error gain, $p_3$ - and a decay constant $p_4$ for the barrier filter.  As such, the space of parameters
\begin{equation}
    \label{eq:synth_params}
    P = [0.2, 5]^3 \times [0.1,200],~p_{1,2,3} \in [0.2,5]^3,~p_4 \in [0.1,200].
\end{equation}

To determine a ``good" parameter set $p \in P$ without knowledge of how these parameters affect the controller $U$, we require a robustness measure $\rho$ (Definition~\ref{def:robustness}) to act as a determiner of satisfactory system performance.  As such, we will use the same robustness measure $\rho$ in equation~\eqref{eq:robustness} in Section~\ref{sec:verification_examples} to facilitate a risk-aware verification comparison between our to-be-calculated controller and the baseline controller for the Robotarium.  To be clear, we will reproduce this robustness measure $\rho$ in its application to our parameterized system trajectories $\mathbf{x}^{\theta,p}$, where $h_g,h_f$ are as defined in equations~\eqref{eq:avoid} and~\eqref{eq:future} respectively:
\begin{gather}
    \rho_g(\mathbf{x}^{\theta,p}) = \min\limits_{t \in [0,30]}~h_g\left( \mathbf{x}^{\theta,p}_t\right), \quad \rho_f(\mathbf{x}^{\theta,p}) = \max\limits_{t \in [0,30]}~h_f\left(\mathbf{x}^{\theta,p}_t\right), \\
    \label{eq:synth_robustness}
    \rho(\mathbf{x}^{\theta,p}) = 
    \begin{cases}
    \rho_g(\mathbf{x}^{\theta,p}), & \mbox{if~} \rho_g(\mathbf{x}^{\theta,p}), \rho_f(\mathbf{x}^{\theta,p}) \geq 0, \\
    \max\left\{\rho_g(\mathbf{x}^{\theta,p}), -0.1 \right\}, & \mbox{if~} \rho_g(\mathbf{x}^{\theta,p}) < 0, \\
    \max\left\{\rho_f(\mathbf{x}^{\theta,p}), -0.1\right\}, & \mbox{else}.
    \end{cases}
\end{gather}

\begin{figure}[t]
    \centering
    \includegraphics{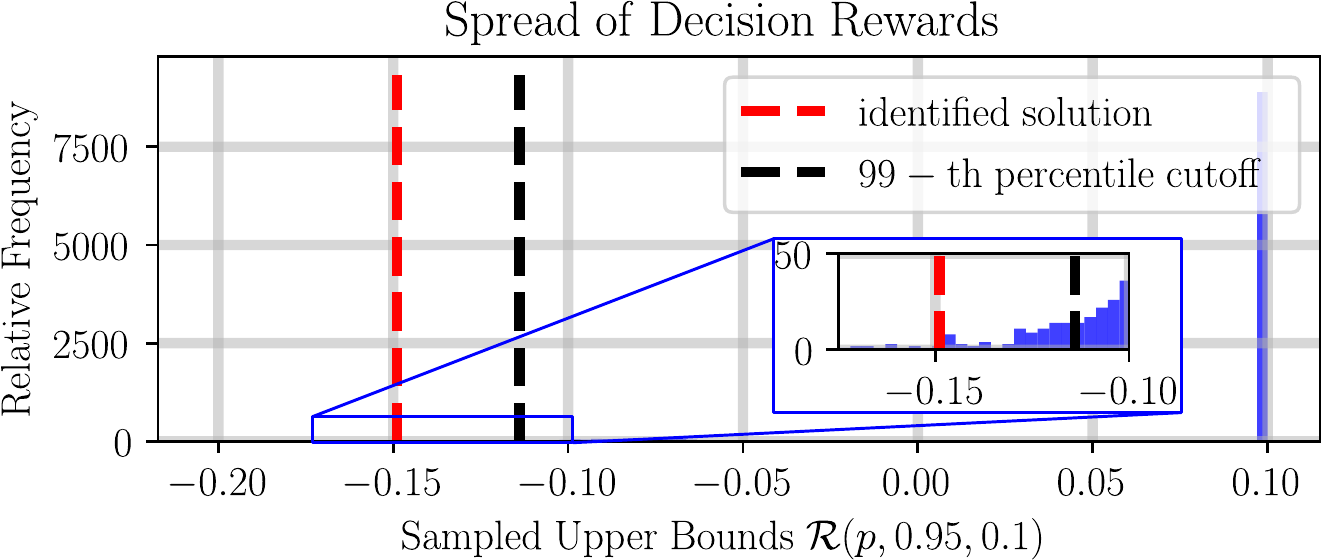}
    \caption{The final use case of our sample-based bounds arises in Risk-Aware Policy Synthesis.  Our goal is to identify a parameterized controller that maximizes a lower bound on worst-case system performance, \textit{i.e.} minimizes an upper bound, $\riskmap(p,0.95,0.1)$, over the set of controller parameters $p \in P$~\eqref{eq:synth_general_opt}.  Shown above in blue is the distribution of this upper bound for all controller parameters $p \in P$ and was generated by taking $20000$ uniform parameter samples $p$ and evaluating $\riskmap(p,0.95,0.1)$.  As per the decision selection process detailed in Section~\ref{sec:decision_selection}, our goal is to identify a controller in the $99$-th percentile with respect to minimization of this upper bound with the true $99$-th percentile cutoff shown in black and all controllers yielding upper bounds to its left lying in the $99$-th percentile.  As can be seen, our identified solution (red) achieves an upper bound in at least the $99$-th percentile.  This serves as numerical confirmation of both Theorem~\ref{thm:good_solutions} and Corollary~\ref{corr:synth_fund_requirement} insofar as we evaluated the minimum number of controllers prescribed, $N = 459$ controllers, to calculate our solution which meets our desired criteria.}
    \label{fig:synth_percentile_calculation}
\end{figure}

\begin{figure}[t]
    \centering
    \includegraphics[width = \textwidth]{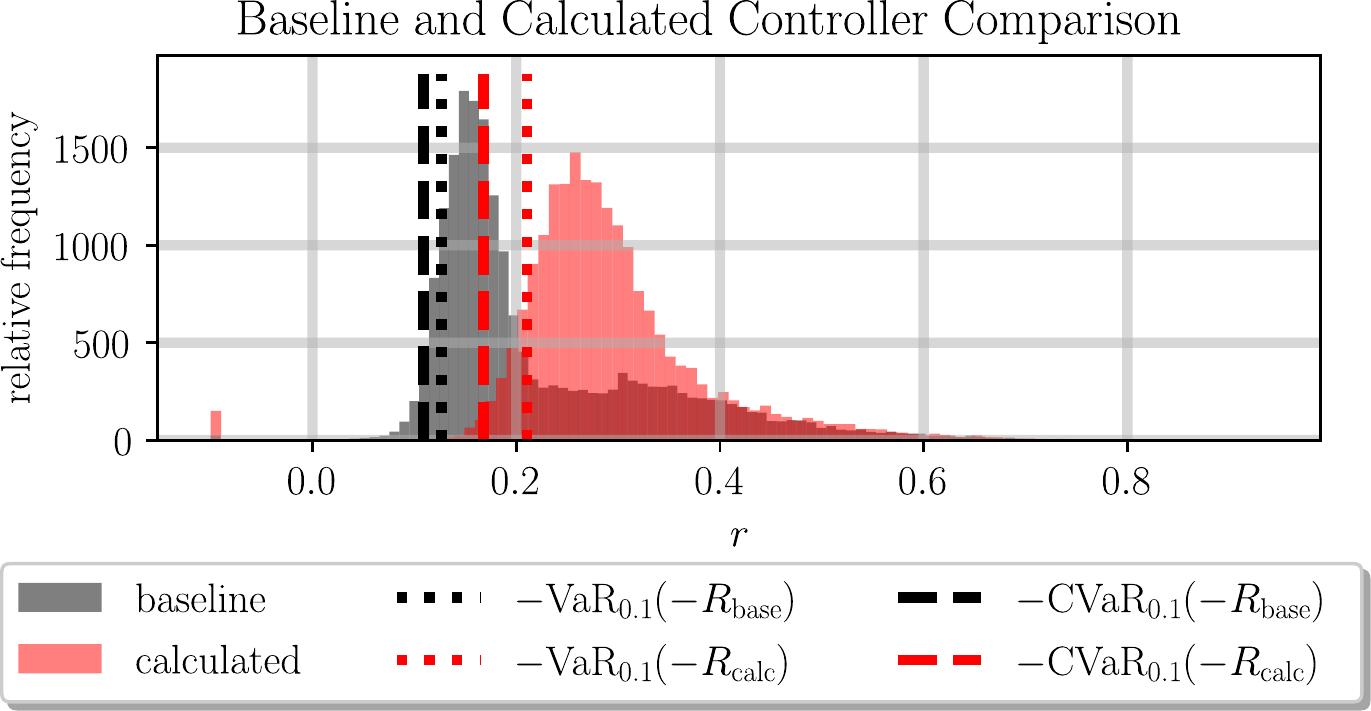}\vspace{-0.1 in}
    \caption{A comparison between the baseline controller provided with the robotarium - the controller that was probabilistically verified in Section~\ref{sec:verification_examples} - and our calculated controller identified in Figure~\ref{fig:synth_percentile_calculation}.  Our risk-aware policy synthesis goal is to identify a controller in the $99$-th percentile with respect to maximizing the lower bound on the expected worst-case system performance in the worst $10\%$ of cases - \textit{i.e.} maximize a lower bound on $-\cvar_{0.1}(-R)$ as expressed in Proposition~\ref{prop:polsynth_motivation}.  Shown above is the distribution of this randomized robustness $R$ for the calculated controller ($R_{\mathrm{calc}}$ in red) and the baseline controller provided with the robotarium ($R_{\mathrm{base}}$ black).  As can be seen, the calculated controller outperforms the baseline controller insofar as the worst-case robustness value for $10\%$ of cases $-\var_{0.1}(-R_{\mathrm{calc}}) \geq -\var_{0.1}(-R_{\mathrm{base}})$, and the expected value in the worst $10\%$ of cases $-\cvar_{0.1}(-R_{\mathrm{calc}}) \geq -\cvar_{0.1}(-R_{\mathrm{base}})$ as well.}
    \label{fig:comparison}
\end{figure}

\noindent Our randomized robustness $R_p$ will also be defined similarly.  Specifically, to take a sample $r_p$ of our randomized robustness $R_p$ as per Definition~\ref{def:random_rob_parameterized}, we first uniformly sample initial conditions and parameters $(x_0,\theta)$ from the space below:
\begin{equation}
    \mathcal{X}_0  = \{\mathbf{x} \in \mathcal{X}^{3}~|~h_g(\mathbf{x}) \geq 0.3 \}, \quad 
    \Theta = \{M\mathbf{x} \in \mathcal{X}^{3}~|~h_g(\mathbf{x}) \geq 0.3 \},
\end{equation}
then we evaluate the robustness of the corresponding state trajectory, \textit{i.e.} $r_p = \rho(\mathbf{x}^{\theta,p})$.  Our goal is to identify a parameter set $p \in P$ that minimizes the upper bound $r^*_p$ for $\cvar_{\alpha=0.1}(-R_p)$ identified with $95\%$ confidence with $N_{\riskmap} = 149$ samples. By definition of  $\riskmap$ in equation~\eqref{eq:riskmap}, this results in the following optimization problem:
\begin{equation}
    \label{eq:polsynth_ex_opt}
    \min_{p \in P}~\riskmap(p,0.95.0.1),~N_{\riskmap} = 149,
\end{equation}
which is of the general form studied in~\eqref{eq:synth_general_opt} with $P$ as in~\eqref{eq:synth_params}.  However, identification of the exact solution to~\eqref{eq:polsynth_ex_opt} will be difficult.  As such, we hope to identify a policy that is in the $99-$th percentile with $99\%$ confidence, \textit{i.e.} $\gamma = 0.99$ and $\epsilon = 0.01$.  

According to Corollary~\ref{corr:synth_fund_requirement}, to generate a parameter set $p_i \in P$ whose controller $U$ is in the $99$-th percentile of all controllers with minimum probability $\gamma = 0.99$, we are required to test $N \geq 459$ uniformly randomly sampled parameters $p \in P$.  After doing so, we identified a parameter set $p_i \in P$ that realized an upper bound $\riskmap(p_i,0.95,0.1) = -0.1489$.  As can be seen in Figure~\ref{fig:synth_percentile_calculation}, the controller corresponding to this policy is indeed in at-least the $99$-th percentile with respect to all controllers achievable via this parameterization.  This serves as numerical confirmation of Corollary~\ref{corr:synth_fund_requirement} and its parent, Theorem~\ref{thm:good_solutions}.  We can also compare the distribution of the robustness of trajectories realized by this controller $U_{\mathrm{calc}}$ as compared to the baseline controller with which the robotarium comes equipped $U_{\mathrm{base}}$.  This information can be seen in Figure~\ref{fig:comparison}.  Notice that our identified controller outperforms the baseline robotarium controller in a risk-aware sense, at-least with respect to the robustness measure of interest.

\section{Conclusion}
\label{sec:conclusion}
To develop a systematic procedure for risk-aware verification, we first developed a procedure to generate sample-based bounds for $g$-entropic risk measures - a large subset of coherent risk measures.  Then, we reframed risk-aware verification as a risk measure identification problem that makes use of our prior results.  We also noted that we could extend these inequalities to the risk-aware policy synthesis setting.  However, identification of such a policy would likely correspond to solving a non-convex optimization problem that would be very difficult to solve.  To simplify risk-aware policy generation then, we developed a sample-based approximation for a general class of optimization problems.  This approximation lets us repeatably and reliably identify decisions that outperform a large fraction of all possible decisions available with respect to optimization of a provided objective.  The prior sample-based bounds then provide the objective function for the risk-aware policy synthesis setting, with a ``good" policy determined via the prior, sample-based approach.

In future work, we first hope to analyze the tightness of the sample-based bounds we generate.  Specifically, as the required confidence in the generated upper-bound increases, the bound tends to converge to the upper bound of the random variable itself.  We also hope to provide lower bounds as well and move towards generating sample-based concentration inequalities for $g$-entropic risk measures.  Second, we hope to apply these tighter bounds to provide less conservative bounds for system verification, as they are quite conservative currently.  In addition, we hope to reduce the number of samples required for the generation of these bounds to facilitate the generation of high-confidence performance bounds on hardware systems.  Finally, we hope to extend the risk-aware synthesis procedure for controller generation on hardware systems.  The inability to sample only those controller parameters that yield safe controllers is the primary hindrance to the immediate application of our approach to hardware systems.  Furthermore, it is not entirely clear that a ``good" risk-aware simulator controller will achieve similarly ``good" performance on hardware.  Here, we think that works addressing the Sim2Real gap might prove useful~\cite{doersch2019sim2real,kadian2020sim2real}.
%% \linenumbers

\section{Acknowledgements}
This work was supported by the AFOSR Test and Evaluation Program, grant FA9550-19-1-0302.

%% main text
\bibliographystyle{elsarticle-num} 
\bibliography{cas-refs}

\end{document}